\documentclass[letterpaper, 12pt]{article} 
\usepackage[margin=1in]{geometry}
\newcounter{somecounter}
\usepackage{booktabs, array}
\usepackage{setspace}
\usepackage{float}
\usepackage{epstopdf}
\usepackage{standalone}
\usepackage{color}
\usepackage{xcolor}
\usepackage{bm}
\usepackage{amsmath,amsfonts,amssymb}
\usepackage{amsthm}

\usepackage{pdflscape}
\usepackage{authblk}
\usepackage{graphicx}
\usepackage[colorlinks=true,citecolor=black,linkcolor=black,hidelinks]{hyperref}
\usepackage[normalem]{ulem}
\usepackage{multirow}
\usepackage{multicol}
\usepackage[flushleft]{threeparttable}
\usepackage{enumitem}
\usepackage{url}
\usepackage[utf8]{inputenc}
\usepackage{amsmath}
\usepackage{amsfonts}
\usepackage{amsthm}

\usepackage{lipsum}
\makeatletter
\renewenvironment{proof}[1][\proofname]{\par
  \pushQED{\qed}%
  \normalfont \topsep6\p@\@plus6\p@\relax
  \list{}{\leftmargin=1em
          \rightmargin=\leftmargin
          \settowidth{\itemindent}{\itshape#1}%
          \labelwidth=\itemindent
          \parsep=0pt \listparindent=\parindent 
  }
  \item[\hskip\labelsep
        \itshape
    #1\@addpunct{.}]\ignorespaces
}{%
  \popQED\endlist\@endpefalse
}
\makeatother
\usepackage{natbib}

\newtheorem{theorem}{Theorem}    
 
\newtheorem{assumption}{Assumption}

\usepackage{mathtools}

\newcommand{\interior}[1]{%
 {\kern0pt#1}^{\mathrm{o}}%
}
\usepackage[english]{babel}
\usepackage[autostyle, english = american]{csquotes}
\MakeOuterQuote{"}

\usepackage{algorithm,algpseudocode}
\usepackage{caption}
\usepackage{subcaption}
\usepackage[makeroom]{cancel}

\makeatletter
\newcommand*\bigcdot{\mathpalette\bigcdot@{.5}}
\newcommand*\bigcdot@[2]{\mathbin{\vcenter{\hbox{\scalebox{#2}{$\m@th#1\bullet$}}}}}
\makeatother

\title{\textbf{A Stein Gradient Descent Approach for Doubly Intractable Distributions}}

\author[1]{Heesang Lee}
\author[1]{Songhee Kim}
\author[3]{Bokgyeong Kang}
\author[1,2]{Jaewoo Park}
\affil[1]{Department of Statistics and Data Science, Yonsei University}
\affil[2]{Department of Applied Statistics, Yonsei University}
\affil[3]{Department of Statistics, Dongguk University}

\begin{document}

\def\spacingset#1{\renewcommand{\baselinestretch}%
{#1}\small\normalsize} \spacingset{1}

\maketitle

\begin{abstract}
Bayesian inference for doubly intractable distributions is challenging because they include intractable terms, which are functions of parameters of interest. Although several alternatives have been developed for such models, they are computationally intensive due to repeated auxiliary variable simulations. We propose a novel Monte Carlo Stein variational gradient descent (MC-SVGD) approach for inference for doubly intractable distributions. Through an efficient gradient approximation, our MC-SVGD approach rapidly transforms an arbitrary reference distribution to approximate the posterior distribution of interest, without necessitating any predefined variational distribution class for the posterior. Such a transport map is obtained by minimizing Kullback-Leibler divergence between the transformed and posterior distributions in a reproducing kernel Hilbert space (RKHS). We also investigate the convergence rate of the proposed method. We illustrate the application of the method to challenging examples, including a Potts model, an exponential random graph model, and a Conway--Maxwell--Poisson regression model. The proposed method achieves substantial computational gains over existing algorithms, while providing comparable inferential performance for the posterior distributions.
\end{abstract}

\noindent%
\vfill

\newpage
\spacingset{1.8} 

\section{Introduction}
\label{sec:1}

Models with intractable normalizing functions arise in a wide variety of statistical models including Potts models \citep{potts1952some} for discrete spatial lattice data, exponential random graph models (ERGMs) \citep{robins2007introduction,park2022bayesian} for network data, and Conway--Maxwell--Poisson regression models \citep{conway1962network,chanialidis2018efficient,Kang2024comp} for count data exhibiting under- and overdispersion. Let $\mathbf{x} \in \mathcal{X}$ denote data generated from a probability model $f(\textbf{x} | \boldsymbol{\theta})$ for a parameter $\boldsymbol{\theta} \in \mathbf{\Theta}\subset \mathbb{R}^d$ with a likelihood function $L(\boldsymbol{\theta}|\textbf{x}) = h(\textbf{x}|\boldsymbol{\theta})/Z(\boldsymbol{\theta})$, where $Z(\boldsymbol{\theta})$ is an intractable normalizing function. The posterior density of $\boldsymbol{\theta}$ is $\pi(\boldsymbol{\theta}) \propto p(\boldsymbol{\theta})h(\mathbf{x}|\boldsymbol{\theta})/Z(\boldsymbol{\theta})$ where $p(\boldsymbol{\theta})$ is a prior density. It is challenging to apply standard Markov chain Monte Carlo (MCMC) algorithms to such models since $Z(\boldsymbol{\theta})$ cannot be analytically evaluated, resulting in so-called doubly intractable posterior distributions. 

Computational approaches have been developed for Bayesian inference for doubly intractable distributions \citep[see][for review]{park2018bayesian}. They either approximate $Z(\boldsymbol{\theta})$ \citep[cf.][]{atchade2013bayesian,lyne2015russian,park2020function} or avoid evaluation of $Z(\boldsymbol{\theta})$ by introducing auxiliary variables \citep[cf.][]{murray2006,liang2010double,liang2016adaptive}. However, these algorithms require repeated simulations of auxiliary datasets from the probability model, which can be computationally intensive \citep{park2018bayesian}. More recent works have been proposed to expedite inference for doubly intractable distributions. \cite{boland2018efficient} developed a fast pre-computing MCMC algorithm that utilizes Monte Carlo samples simulated at a grid over $\bm{\Theta}$ prior to the implementation of the algorithm. \cite{park2020function} replaced expensive importance sampling estimation of $Z(\boldsymbol{\theta})$ with Gaussian process emulation. 
\cite{10.1214/18-BA1130} introduced a deterministic function, dependent on the sufficient statistics of the model, to replace the intractable likelihood function.
\cite{vu2023warped} replaced the likelihood function with a warped gradient-enhanced Gaussian process surrogate model. Although these approaches can greatly enhance computational efficiency, their performance is heavily dependent on the similarity between the true models and their surrogates. Besides, it is important to choose appropriate design points for constructing effective surrogate models, which is generally challenging for multidimensional $\mathbf{\Theta}$. \cite{tan2020bayesian} proposed variational inference methods for ERGMs, where posterior distributions are approximated by Gaussian distributions. They optimize the evidence lower bound (ELBO) through the stochastic gradient ascent, where the gradients are replaced with their consistent estimates. While the methods can be extremely fast, they require considerable tuning to achieve optimal performance. For example, they require the adjusted maximum pseudolikelihood estimate (MPLE) \citep{bouranis2018bayesian} or Markov chain Monte Carlo maximum likelihood estimate (MCMC-MLE) \citep{geyer1992constrained} as inputs. Furthermore, they are designed exclusively for ERGMs and cannot be easily extended to other models. 

In this manuscript, we propose a Stein variational gradient descent (SVGD) approach for doubly intractable distributions and provide its theoretical justification. The SVGD method introduced by \citet{liu2016stein} is a variational inference approach that transforms an arbitrary reference distribution to a target distribution, without the need to assume any specific variational distribution class for the target. Due to its flexibility and computational efficiency, the SVGD approach has been extensively applied in the deep learning literature \citep{pu2017vae,haarnoja2017reinforcement,jaini2021learning,kim2023density}. 
However, this approach is not suitable for doubly intractable distributions, as it necessitates evaluating the gradient of the target posterior distribution. To address this limitation, we develop a Monte Carlo SVGD (MC-SVGD) method that efficiently approximates the intractable gradient throughout the process. We introduce several computational techniques to maintain the computational efficiency of the original SVGD while extending its applicability to doubly intractable distributions. These innovations ensure that the method remains both computationally efficient and adaptable to more complex models.
Furthermore, we derive a finite particle convergence rate for MC-SVGD. Specifically, when the target posterior $\pi(\boldsymbol{\theta})$ is sub-Gaussian with a Lipschitz score, our MC-SVGD method, using $n$ particles and $m$ Monte Carlo samples, achieves the kernel Stein discrepancy (KSD) convergence to zero at a rate of $\mathcal{O}((\log \log (n))^{-1/2}) + \mathcal{O}(m^{-1/2})$ rate. To the best of our knowledge, this is the first attempt to investigate the convergence rate of SVGD when the target distribution is intractable. 

The remainder of this manuscript is organized as follows.  In Section~\ref{sec:2}, we provide a brief overview of the SVGD method introduced by \citet{liu2016stein}. In Section~\ref{sec:3} we present our MC-SVGD approach, which extends its applicability to doubly intractable distributions. We offer a theoretical justification for the proposed method. In Section~\ref{sec:4}, we demonstrate the application of our method to three challenging examples and investigate its performance through comparisons with existing techniques. We conclude in Section~\ref{sec:5} with a comprehensive summary of our findings, along with a discussion of potential future research directions.

\section{Background}
\label{sec:2}

In this section, we introduce notations used throughout this manuscript. We briefly explain a Stein variational gradient descent introduced by \citet{liu2016stein}, which is a flexible variational inference method for approximating target distributions. 

\subsection{Notation}
\label{sec:notation}

Let $k : \mathbf{\Theta} \times \mathbf{\Theta} \rightarrow \mathbb{R}$ be a positive definite kernel. The reproducing kernel Hilbert space (RKHS) $\mathcal{H}$ of $k(\boldsymbol{\theta}, \boldsymbol{\theta}^{*})$ is the closure of linear span $\{f:f(\boldsymbol{\theta})=\sum_{i=1}^m a_i k(\boldsymbol{\theta}, \boldsymbol{\theta}_i), a_i \in \mathbb{R}, m \in \mathbb{N}, \boldsymbol{\theta}_i \in \mathbf{\Theta}\}$, equipped with inner products $\langle f,g\rangle_{\mathcal{H}} = \sum_{ij} a_ib_jk(\boldsymbol{\theta}_i, \boldsymbol{\theta}_j)$ for $g(\boldsymbol{\theta})=\sum_i b_ik(\boldsymbol{\theta}, \boldsymbol{\theta}_i)$. A norm is defined from the inner product as  $\|f\|_{\mathcal{H}}^2=\langle f,f\rangle_{\mathcal{H}}$. The operator norm of the function is defined as $\|f\|_{op}$. For a vector function $\boldsymbol{f}=(f_1,\cdots,f_d)'$ with $f_i \in \mathcal{H}$, we can define a function space $\mathcal{H}^d$, equipped with inner products $\langle \boldsymbol{f},\boldsymbol{g}\rangle_{\mathcal{H}^d} = \sum_{i=1}^d \langle f_i,g_i\rangle_{\mathcal{H}}$. Similarly a norm is defined as $\|\boldsymbol{f}\|_{\mathcal{H}^d}^2=\langle \boldsymbol{f},\boldsymbol{f}\rangle_{\mathcal{H}^d}$. A derivative of the vector function is defined as $\nabla_{\boldsymbol{\theta}} \boldsymbol{f}=(\nabla_{\boldsymbol{\theta}} f_1,\cdots,\nabla_{\boldsymbol{\theta}} f_d)$. Let $\mathbf{\Pi}_1$ be the set of probability distribution on $\mathbb{R}^d$ with integrable first moments and $\Pi \in \mathbf{\Pi}_1$ be a target posterior distribution whose density is $\pi(\boldsymbol{\theta})$.


\subsection{Stein Variational Gradient Descent}
\label{sec:svgdback}

Variational methods \citep{bishop2006pattern, blei2017variational} have been widely used to approximate the posterior distribution by minimizing the Kullback-Leibler (KL) divergence between $\Pi$ and a tractable distribution class. The distribution class should be large enough to cover a wide range of target distributions. Furthermore, it should consist of simple distributions that are easy to infer, making the computation process straightforward. The performance of the variational methods largely depends on the choice of the distribution class. Selecting an appropriate class of distributions is challenging in practice. To address this, \cite{liu2016stein} developed an SVGD algorithm, a variational method focusing on distributions obtained by smooth transformations from a tractable reference distribution $\mathcal{Q}$. 
This method does not require specific parametric forms of the variational distribution class and has simple deterministic updates similar to the typical gradient descent algorithm.

Consider an incremental transformation $\mathcal{T}(\boldsymbol{\theta}) = \boldsymbol{\theta} + \epsilon \phi (\boldsymbol{\theta}),$ where $\phi(\boldsymbol{\theta})$ is a smooth function for perturbation direction and
$\epsilon$ is a step size. Let $\Phi(\mathcal{Q})$ be a pushforward distribution (i.e., updated reference distribution) obtained by the transformation $\mathcal{T}(\boldsymbol{\theta})$ when $\boldsymbol{\theta}\sim\mathcal{Q}$. The goal is to find the direction $\phi$ that minimizes the KL divergence between the pushforward distribution $\Phi(\mathcal{Q})$ and the target distribution $\Pi$. \cite{liu2016stein} show that the derivative of $\text{KL}$ divergence between $\Phi(\mathcal{Q})$ and $\Pi$ can be expressed as
$$\nabla_{\epsilon} \text{KL}(\Phi(\mathcal{Q}) \| \Pi) |_{\epsilon=0} = - \mathbb{E}_{\mathcal{Q}} \left[\text{trace}(\mathcal{A}_{\Pi} \phi(\boldsymbol{\theta}))\right],$$
where $\mathcal{A}_{\Pi}\phi(\boldsymbol{\theta}) = \phi(\boldsymbol{\theta}) \nabla_{\boldsymbol{\theta}} \log \pi(\boldsymbol{\theta})' + \nabla_{\boldsymbol{\theta}} \phi(\boldsymbol{\theta})$ is called the Stein operator. Minimizing the KL divergence between $\Phi$ and $\Pi$ is equivalent to 
\begin{equation}
\label{opti_eq}
    \max_{\phi \in \mathcal{H}} \mathbb{E}_{\mathcal{Q}}\left[\text{trace}(\phi(\boldsymbol{\theta}) \nabla_{\boldsymbol{\theta}} \log \pi(\boldsymbol{\theta})' + \nabla_{\boldsymbol{\theta}} \phi(\boldsymbol{\theta}))\right].
\end{equation}
\cite{liu2016stein} assume that $\mathcal{H}$ is a ball of RKHS such as $\{\phi \in \mathcal{H}^d: \|\phi\|_{\mathcal{H}^d}^2 \leq \text{KSD}(\mathcal{Q}, \Pi) \}$ where $\text{KSD}(\mathcal{Q}, \Pi) =\{\max_{\phi \in \mathcal{H}^d} \left[\mathbb{E}_{\mathcal{Q}}  (\mathcal{A}_{\Pi} \phi(\boldsymbol{\theta}))\right]^2 \: \text{s.t.} \: \|\phi\|_{\mathcal{H}^d} \leq 1\}$ is the kernel Stein discrepancy (KSD).
Under the assumption, the direction $\phi$ maximizing the expectation of \eqref{opti_eq} has the following closed form solution: 
 \begin{equation}
 \label{dir_steep}
     \phi^{*}_{\mathcal{Q},\Pi}(\cdot) = \mathbb{E}_{\mathcal{Q}} \left[\nabla_{\boldsymbol{\theta}} \log \pi(\boldsymbol{\theta}) k(\boldsymbol{\theta},\cdot) + \nabla_{\boldsymbol{\theta}} k(\boldsymbol{\theta},\cdot)\right].
 \end{equation}
The closed form \eqref{dir_steep} of the perturbation direction enables a sequence of transformations from a reference distribution $\mathcal{Q}$ to the target distribution $\Pi$. Here, we will refer to the transformation of $\mathcal{Q}$ using \eqref{dir_steep} as continuous SVGD \citep{liu2017stein, shi2024finite}. In \eqref{dir_steep}, one can use commonly used kernels such as the radial basis function (RBF) kernel $k(\boldsymbol{\theta}, \boldsymbol{\theta}^{*}) = \exp(-|| \boldsymbol{\theta} - \boldsymbol{\theta}^{*} ||^2/h)$. Following \cite{liu2016stein}, we set $h=\text{med}^2/\log n$, where  $\text{med}$ denotes the median of the pairwise distance between particles.


In general, the expectation in \eqref{dir_steep} cannot be evaluated analytically.
Given a set of particles $\{\boldsymbol{\theta}_i\}_{i=1}^n $ generated from the reference distribution $\mathcal{Q}$,
\cite{liu2016stein} replace the expectation with an empirical mean as follows:
\begin{equation}
\label{appro_dir}
    \phi_{n}^{*}(\boldsymbol{\theta})=\frac{1}{n} \sum_{i=1}^n [\nabla_{\boldsymbol{\theta}_i} \log \pi(\boldsymbol{\theta}_i) k(\boldsymbol{\theta}_i, \boldsymbol{\theta})  + \nabla_{\boldsymbol{\theta}_i} k(\boldsymbol{\theta}_i,\boldsymbol{\theta})].
\end{equation}
Let $\mathcal{Q}_n = \sum_{i=1}^n \delta_{\boldsymbol{\theta}_i}/n$ denote the sample distribution of the particles $\{\boldsymbol{\theta}_i\}_{i=1}^n$, where $\delta_{\boldsymbol{\theta}}$ denotes the Dirac delta function.
An empirical transport map can be defined as $\mathcal{T}_{n}(\boldsymbol{\theta})=\boldsymbol{\theta}+\epsilon \phi_{n}^{*}(\boldsymbol{\theta})$, which updates the particles and provides the associated pushforward distribution $\Phi_n(\mathcal{Q}_n)$. 
A set of initial particles $\{\boldsymbol{\theta}_i^{(0)}\}_{i=1}^n$ can be generated from the initial reference distribution $\mathcal{Q}^{(0)}$ (e.g., prior).
The above procedure is summarized in Algorithm~\ref{svgdalg}.

\begin{algorithm}
    \caption{Stein Variational Gradient Descent (SVGD) algorithm \citep{liu2016stein} }
\textbf{Input}: A target distribution with a density function $\pi(\boldsymbol{\theta})$ and a set of initial particles $\{\boldsymbol{\theta}_i^{(0)}\}_{i=1}^n$. \\
\textbf{Output}: A set of particles $\left\{\boldsymbol{\theta}_i\right\}_{i=1}^n$ that approximates the target. 

    \begin{algorithmic}[1]
\For {iteration $t$}
\medskip
    \State $\phi_{n}^{*}(\boldsymbol{\theta}_i^{(t)}) \gets \frac{1}{n} \sum_{j=1}^n [\nabla_{\boldsymbol{\theta}_j^{(t)}} \log \pi(\boldsymbol{\theta}_j^{(t)}) k(\boldsymbol{\theta}_j^{(t)}, \boldsymbol{\theta}_i^{(t)})  + \nabla_{\boldsymbol{\theta}_j^{(t)}} k(\boldsymbol{\theta}_j^{(t)},\boldsymbol{\theta}_i^{(t)})]\text{ for } i=1, \cdots, n$ 
    \medskip
    \State $\boldsymbol{\theta}_i^{(t+1)} \gets \boldsymbol{\theta}_i^{(t)} + \epsilon^{(t)} \phi_{n}^{*}(\boldsymbol{\theta}_i^{(t)}) \text{ for } i=1, \cdots, n$
    \medskip
\EndFor
    \end{algorithmic}
     \label{svgdalg}
\end{algorithm}

Note that if we use a single particle, the algorithm collapses to a typical gradient ascent method for finding a maximum a posteriori (MAP) estimate for any kernel that satisfies $\nabla_{\boldsymbol{\theta}} k(\boldsymbol{\theta}, \boldsymbol{\theta})=0$. 
The term $\nabla_{\boldsymbol{\theta}_i} \log \pi(\boldsymbol{\theta}_i) k(\boldsymbol{\theta}_i, \boldsymbol{\theta})$ of \eqref{appro_dir} forces the particles to move towards regions with high values of $\pi(\boldsymbol{\theta})$
using gradient information. 
This is a sum of the gradients weighted by a kernel function;
if $\boldsymbol{\theta}_i$ is close to the remaining particles, a high weight is assigned in the update. The term $\nabla_{\boldsymbol{\theta}_i} k(\boldsymbol{\theta}_i,\boldsymbol{\theta})$ of \eqref{appro_dir} encourages the particles to explore diverse regions by preventing them from collapsing into a single local mode of $\pi(\boldsymbol{\theta})$. Compared to traditional MCMC algorithms, SVGD achieves particle efficiency by using deterministic updates directed toward the steepest descent in reducing the KL divergence \citep{liu2017stein}.

\section{Stein Variational Gradient Descent for Doubly Intractable Distributions}
\label{sec:3}

To expand its applicability to double intractable distributions, we propose a Monte Carlo Stein variational gradient descent method and offer practical guidelines for its implementation. We also provide theoretical justification
for the method.

\subsection{Monte Carlo Stein Variational Gradient Descent}
\label{sec:mcsvgd}

SVGD is particle efficient and does not require specifying a variational distribution class.
However, a direct application of this method to doubly intractable posterior distributions is challenging. Specifically, it requires evaluating the gradient of the log posterior as expressed in \eqref{appro_dir}, which is infeasible due to the intractable normalizing function $Z(\boldsymbol{\theta})$. To address this, we replace the intractable gradient with its Monte Carlo (MC) approximation. The gradient of the log posterior is written as 
\begin{equation}
\nabla_{\boldsymbol{\theta}}\log \pi(\boldsymbol{\theta}) = \nabla_{\boldsymbol{\theta}} \log h(\textbf{x}|\boldsymbol{\theta})- \nabla_{\boldsymbol{\theta}} \log Z(\boldsymbol{\theta}) + \nabla_{\boldsymbol{\theta}}\log p(\boldsymbol{\theta}).
\label{logposteriorgrad}
\end{equation}
Since we have $\nabla_{\boldsymbol{\theta}} \log Z(\boldsymbol{\theta})=\mathbb{E}_{f(\cdot| \boldsymbol{\theta})}\left[ \nabla_{\boldsymbol{\theta}} \log h(\textbf{x}|\boldsymbol{\theta})\right]$, a MC approximation to \eqref{logposteriorgrad} is given by
\begin{equation}
\nabla_{\boldsymbol{\theta}} \log h(\textbf{x}|\boldsymbol{\theta})- \frac{1}{m} \sum_{k=1}^m \nabla_{\boldsymbol{\theta}} \log h(\textbf{y}_k|\boldsymbol{\theta}) + \nabla_{\boldsymbol{\theta}}\log p(\boldsymbol{\theta}),
\label{MCapprox}
\end{equation}
where $\lbrace\textbf{y}_k\rbrace_{k=1}^m$ are MC samples generated from the probability model $f(\cdot | \boldsymbol{\theta})$. Plugging in this MC approximation into Algorithm~\ref{svgdalg} yields a naive version of our proposed approach. This naive method requires simulating auxiliary variables $\lbrace\textbf{y}_k\rbrace_{k=1}^m$ from the probability model every iteration, which may incur significant computational costs.

To accelerate computation, we employ self-normalized importance sampling (SNIS) \citep{tan2020bayesian,park2020function}. For any ${\boldsymbol{\psi}} \in \mathbf{\Theta}$, we have 
\[
\mathbb{E}_{f(\cdot|\boldsymbol{\theta})}\left[ \nabla_{\boldsymbol{\theta}} \log h(\textbf{x}|\boldsymbol{\theta})\right] = \frac{\mathbb{E}_{f(\cdot|\boldsymbol{\psi})}\left[ \nabla_{\boldsymbol{\theta}} \log h(\textbf{x}|\boldsymbol{\theta}) \frac{h(\textbf{x}|\boldsymbol{\theta})}{h(\textbf{x}|{\boldsymbol{\psi}})}\right]} {\mathbb{E}_{f(\cdot|\boldsymbol{\psi})}\left[ \frac{h(\textbf{x}|\boldsymbol{\theta})}{h(\textbf{x}|{\boldsymbol{\psi}})}\right]}.
\]
A consistent estimate of $\mathbb{E}_{f(\cdot|\boldsymbol{\theta})}\left[ \nabla_{\boldsymbol{\theta}} \log h(\textbf{x}|\boldsymbol{\theta})\right]$ can be obtained by
\begin{equation}
\label{snis_eq}
\sum_{k=1}^m w_k \nabla_{\boldsymbol{\theta}} \log h(\textbf{y}_k|\boldsymbol{\theta}) \:\: \text{with} \:\:w_k=\frac{h(\textbf{y}_k|\boldsymbol{\theta})}{h(\textbf{y}_k|{\boldsymbol{\psi}})}/\sum_{k=1}^{m} \frac{h(\textbf{y}_k|\boldsymbol{\theta})}{h(\textbf{y}_k|{\boldsymbol{\psi}})},
\end{equation}
where $\lbrace\textbf{y}_k\rbrace_{k=1}^m$ are MC samples generated from the model $f(\cdot | \boldsymbol{\psi})$. We can plug in this SNIS estimate into \eqref{logposteriorgrad} to approximate the gradient of the log posterior. Once we have MC samples from $f(\cdot | \boldsymbol{\psi})$, we can reuse these samples to approximate $\mathbb{E}_{f(\cdot|\boldsymbol{\theta})}\left[ \nabla_{\boldsymbol{\theta}} \log h(\textbf{x}|\boldsymbol{\theta})\right]$ for any $\boldsymbol{\theta}$. 
Therefore, SNIS can reduce computation costs significantly compared to the straightforward MC approximation.

It is important to note that the performance of SNIS can be influenced by the choice of $\boldsymbol{\psi}$. An approximation to the MLE or MAP estimate of $\boldsymbol{\theta}$ are typically used for $\boldsymbol{\psi}$. However, SNIS may produce poor estimates when $\boldsymbol{\psi}$ is far from $\boldsymbol{\theta}$. To address this problem, \cite{tan2020bayesian} proposed an adaptive sampling method that uses the effective sample size (ESS), calculated from the SNIS weights $w_k$'s, to guide whether to employ MC approximation or continue with SNIS  \citep{kong1994sequential, martino2017effective}. 
If the ESS falls below a prespecified threshold value---indicative of poor SNIS performance, they generate auxiliary variables from $f(\cdot | \boldsymbol{\theta})$ and estimate the intractable term by MC approximation. Otherwise, SNIS is used for approximation, leveraging previously generated auxiliary samples associated with $\boldsymbol{\psi}$.
This adaptive method can avoid poor estimates from SNIS while reducing computational costs compared to the straightforward MC approximation. Our Monte Carlo SVGD (MC-SVGD) incorporates this adaptive strategy to ensure reliability and computational efficiency.


Let $\mathbb{C}^{(t)}$ denote a set that contains cumulative information about particles and corresponding MC samples up to the $t$th iteration. We initialize $\boldsymbol{\psi}^{(0)}$ as the MAP estimate for $\boldsymbol{\theta}$ and $\lbrace\textbf{y}_{k}^{(0)}\rbrace_{k=1}^m$ as MC samples generated from $f(\cdot | \boldsymbol{\psi}^{(0)})$. Note that we can obtain the MAP by a preliminary run of MC-SVGD with a single particle; this step only takes several seconds. 
Given $\{\boldsymbol{\theta}_i^{(t)}\}_{i=1}^n$ and $\mathbb{C}^{(t)}$, for each particle $\boldsymbol{\theta}_i^{(t)}$, we find $\boldsymbol{\psi} \in \mathbb{C}^{(t-1)}$ closest to $\boldsymbol{\theta}_i^{(t)}$ in Euclidean distance and associated MC samples $\lbrace\textbf{y}_{k}\rbrace_{k=1}^m \in \mathbb{C}^{(t-1)}$ (i.e., $\lbrace\textbf{y}_{k}\rbrace_{k=1}^m \sim f(\cdot | \boldsymbol{\psi})$ from a past iteration) and compute $\text{ESS}$ $ = 1/\sum_{k=1}^m (w_{ik}^{(t)})^2$. If the ESS is greater than or equal to the threshold value, we approximate $\mathbb{E}_{f(\cdot|\boldsymbol{\theta}_i^{(t)})}\big[ \nabla_{\boldsymbol{\theta}} \log h(\textbf{x}|\boldsymbol{\theta}_i^{(t)})\big]$ through SNIS as in \eqref{snis_eq}. Otherwise, we generate "new" MC samples $\lbrace\textbf{y}_{k}\rbrace_{k=1}^m \sim f(\cdot | \boldsymbol{\theta}_i^{(t)})$ and perform a straightforward MC approximation as in \eqref{MCapprox}. We set $\boldsymbol{\psi}^{(t)}_i=\boldsymbol{\theta}_i^{(t)}$ and $ \lbrace\textbf{y}_{ik}^{(t)}\rbrace_{k=1}^{m}=\lbrace\textbf{y}_{k}\rbrace_{k=1}^m$, and append them to $\mathbb{C}^{(t)}$ (i.e., $\mathbb{C}^{(t)} \gets \mathbb{C}^{(t)} \cup \lbrace\boldsymbol{\psi}_i^{(t)}, \lbrace\textbf{y}_{ik}^{(t)}\rbrace_{k=1}^m\rbrace$). This procedure is repeated for each particle. Then we set $\mathbb{C}^{(t+1)} = \mathbb{C}^{(t)}$.

Similar to SVGD, we approximate the direction using the empirical mean as
\begin{equation}
\label{appro_dir_mcsvgd}
    \phi_{n,m}^{*}(\boldsymbol{\theta})=\frac{1}{n} \sum_{i=1}^n [\nabla_{\boldsymbol{\theta}_i} \log \widehat{\pi}(\boldsymbol{\theta}_i) k(\boldsymbol{\theta}_i, \boldsymbol{\theta})  + \nabla_{\boldsymbol{\theta}_i} k(\boldsymbol{\theta}_i,\boldsymbol{\theta})],
\end{equation}
where $\nabla_{\boldsymbol{\theta}_i}\log \widehat{\pi}(\boldsymbol{\theta}_i)$ is either a SNIS estimate or a straightforward MC approximation, and $\mathcal{Q}_{n,m} = \sum_{i=1}^n \delta_{\boldsymbol{\theta}_i}/n$ is the sample distribution of the particles $\{\boldsymbol{\theta}_i\}_{i=1}^n$. Note that $\mathcal{Q}_{n,m}$ is the sample distribution associated with MC-SVGD, while $\mathcal{Q}_{n}$ is associated with SVGD. We define an empirical transport map as $\mathcal{T}_{n,m}(\boldsymbol{\theta})=\boldsymbol{\theta}+\epsilon \phi_{n,m}^{*}(\boldsymbol{\theta})$ for updating particles. We denote the associated pushforward distribution by $\Phi_{n,m}(\mathcal{Q}_{n,m})$. The MC-SVGD is summarized in Algorithm~\ref{mcsvgdalg}.

\begin{algorithm}
    \caption{Monte Carlo Stein Variational Gradient Descent (MC-SVGD) algorithm}
\textbf{Input}: A sampler from the probability model $f(\cdot | \boldsymbol{\theta})$ and a set of initial particles $\{\boldsymbol{\theta}_i^{(0)}\}_{i=1}^n$. \\
 \textbf{Output}: A set of particles $\left\{\boldsymbol{\theta}_i\right\}_{i=1}^n$ that approximates the target.
    \begin{algorithmic}[1]
\For {iteration $t$}
 \For {particle $i$}
    \State $\boldsymbol{\psi} \gets$ closest to $\boldsymbol{\theta}_i^{(t)}$ in the collection $\mathbb{C}^{(t-1)}$
    \medskip
    \State $\lbrace\textbf{y}_{k}\rbrace_{k=1}^m \gets$  MC samples associated with $\boldsymbol{\psi}$ in the collection $\mathbb{C}^{(t-1)}$ 
    \medskip
    \State $w_{ik}^{(t)} \gets \frac{h(\textbf{y}_{k}|\boldsymbol{\theta}_i^{(t)})}{h(\textbf{y}_{k}|{\boldsymbol{\psi}})}/\sum_{k=1}^{m} \frac{h(\textbf{y}_{k}|\boldsymbol{\theta}_i^{(t)})}{h(\textbf{y}_{k}|{\boldsymbol{\psi}})}\text{ for } k=1, \cdots, m$
    \medskip
    \State $\text{ESS}$ $ \gets 1/\sum_{k=1}^m (w_{ik}^{(t)})^2$
    \medskip
\If{$\text{ESS}$ $\geq$ threshold}        
\medskip
     \State 
$\nabla_{\boldsymbol{\theta}_i} \log \widehat{\pi}(\boldsymbol{\theta}_i^{(t)}) 
     \gets 
     \begin{aligned}[t] & \nabla_{\boldsymbol{\theta}_i} \log h(\textbf{x}|\boldsymbol{\theta}_i^{(t)})-\sum_{k=1}^m {w_{ik}^{(t)}} \nabla_{\boldsymbol{\theta}_i} \log h(\textbf{y}_{k}|\boldsymbol{\psi}) + \nabla_{\boldsymbol{\theta}_i}\log p(\boldsymbol{\theta}_i^{(t)})\end{aligned}$
\Else
\medskip
      \State $\lbrace\textbf{y}_{k}\rbrace_{k=1}^m \sim f(\cdot | \boldsymbol{\theta}_i^{(t)})$
      \medskip
      \State $\nabla_{\boldsymbol{\theta}_i} \log \widehat{\pi}(\boldsymbol{\theta}_i^{(t)}) \gets 
      \begin{aligned}[t] & \nabla_{\boldsymbol{\theta}_i} \log h(\textbf{x}|\boldsymbol{\theta}_i^{(t)})-\frac{1}{m}\sum_{k=1}^m \nabla_{\boldsymbol{\theta}_i} \log h(\textbf{y}_{k}|\boldsymbol{\theta}_i^{(t)}) + \nabla_{\boldsymbol{\theta}_i}\log p(\boldsymbol{\theta}_i^{(t)})\end{aligned}$
      \State $\boldsymbol{\psi}^{(t)}_i \gets \boldsymbol{\theta}_i^{(t)}$ 
      \medskip
      \State $\lbrace\textbf{y}_{ik}^{(t)}\rbrace_{k=1}^m \gets \lbrace\textbf{y}_k\rbrace_{k=1}^m$ 
      \medskip
       \State $\mathbb{C}^{(t)} \gets \mathbb{C}^{(t)} \cup \lbrace\boldsymbol{\psi}_i^{(t)}, \lbrace\textbf{y}_{ik}^{(t)}\rbrace_{k=1}^m\rbrace$
     \medskip
\EndIf
\EndFor
\medskip
       \State $\mathbb{C}^{(t+1)} \gets \mathbb{C}^{(t)}$
      \medskip
    \State $\phi_{n,m}^{*}(\boldsymbol{\theta}_i^{(t)}) \gets \frac{1}{n} \sum_{j=1}^n [\nabla_{\boldsymbol{\theta}_j^{(t)}} \log \widehat{\pi}(\boldsymbol{\theta}_j^{(t)}) k(\boldsymbol{\theta}_j^{(t)}, \boldsymbol{\theta}_i)  + \nabla_{\boldsymbol{\theta}_j^{(t)}} k(\boldsymbol{\theta}_j^{(t)},\boldsymbol{\theta}_i)]\text{ for } i=1, \cdots, n$ 
    \medskip
    \State $\boldsymbol{\theta}_i^{(t+1)} \gets \boldsymbol{\theta}_i^{(t)} + \epsilon^{(t)} \phi_{n,m}^{*}(\boldsymbol{\theta}_i^{(t)})\text{ for } i=1, \cdots, n$
    \medskip
\EndFor
    \end{algorithmic}
        \label{mcsvgdalg}
\end{algorithm}

In comparison to SVGD, MC-SVGD involves additional computations to approximate the gradient of the log posterior. This additional computation can be efficiently managed through parallel computation since the approximations can be carried out independently for each particle. The computational costs can be reduced by a factor of the number of available cores. The parallel computing was implemented using {\tt OpenMP} library in {\tt C++} or {\tt doParallel} package in {\tt R}.

To determine the suitable number of iterations for our method, we employ the KL divergence between $\mathcal{Q}_{n,m}^{(t)}$ and $\mathcal{Q}_{n,m}^{(t-100)}$ where $\mathcal{Q}_{n,m}^{(t)}$ denote $\mathcal{Q}_{n.m}$ at iteration $t$. We assess the KL divergence every 100 iterations and stop the algorithm when the divergence stabilizes without significant changes.

\subsection{Implementation Details}
\label{sec:implementation}

This section provides essential implementation guidelines for MC-SVGD. To effectively implement our method, the following parameters must be tuned: (1) a set of initial particles, (2) step size $\epsilon$, (3) the number $n$ of particles, (4) the number $m$ of MC samples, and (5) ESS threshold.
We draw initial particles $\lbrace\boldsymbol{\theta}_i^{(0)}\rbrace_{i=1}^n$ from $N(\boldsymbol{\theta}^{\text{MAP}}, \boldsymbol{\Sigma})$ where $\boldsymbol{\theta}^{\text{MAP}}$ is the MAP estimate and $\boldsymbol{\Sigma}$ can be set as an identity matrix or determined by a low-cost estimate from MPLE or a simplified modeling approach (see Section~\ref{sec:4} for further details).
We recommend $\epsilon$ between 0.0001 and 0.001 to satisfy the condition associated with the step size specified in Theorem~\ref{Theorem 3} in Section~\ref{sec:theory}. More details are available in the same section.

We conduct sensitivity analyses for the remaining tuning components. Full details on the analysis are provided in the supplemental Section B.
Choosing appropriate values for $n$ and $m$ needs a consideration of the trade-off between accuracy and computational efficiency. According to Theorem~\ref{Theorem 3}, increasing $n$ and $m$ reduces the approximation error at a rate of $\mathcal{O}((\log \log (n))^{-1/2})+\mathcal{O}(m^{-1/2})$. Our empirical studies recommend $n \geq 30d$ and $m \geq 50$ to achieve reliable inference results. 
We found that an ESS threshold of $m/1.5$ ensures reliable SNIS estimates. Note that this threshold is more conservative than the $m/3$ used in \cite{tan2020bayesian}.

\subsection{Theoretical Justification for MC-SVGD}
\label{sec:theory}

We examine the approximation error for the proposed MC-SVGD approach which involves two sources of error: (1) finite particle approximation error and (2) MC approximation error. Taking both types of error into account, we provide a bound for KSD between MC-SVGD and the target distribution. 
We make the following assumptions.

\begin{assumption}[Lipschitz, mean-zero score function]
 \label{Assumption 1}
  The target distribution $\Pi$ $\in \mathbf{\Pi}_1$ has a differentiable density $\pi(\boldsymbol{\theta})$ with an L-Lipschitz score function $\nabla \log \pi(\boldsymbol{\theta}), i.e., \|\nabla \log \pi(\boldsymbol{\theta})-\nabla \log \pi(\boldsymbol{\theta}^{*})\|_2 \leq L\|\boldsymbol{\theta}-\boldsymbol{\theta}^{*}\|_2$ for all $\boldsymbol{\theta},\boldsymbol{\theta}^{*} \in \mathbb{R}^d$. Moreover, $\mathbb{E}_{\Pi}[\nabla \log \pi(\boldsymbol{\theta})]=0$ and $\nabla \log \pi(\boldsymbol{\theta})=0$ for some $\boldsymbol{\theta} \in \mathbb{R}^d$.
\end{assumption}

\begin{assumption}[Bounded kernel derivatives]
 \label{Assumption 2}
 The kernel $k$ is twice differentiable and \\
 $\sup_{\boldsymbol{\theta}, \boldsymbol{\theta}^{*} \in \mathbb{R}^d} \max \{|k(\boldsymbol{\theta},\boldsymbol{\theta}^{*})|, \|\nabla_{\boldsymbol{\theta}} k(\boldsymbol{\theta},\boldsymbol{\theta}^{*})\|_2, \|\nabla_{\boldsymbol{\theta}^{*}} \nabla_{\boldsymbol{\theta}} k(\boldsymbol{\theta},\boldsymbol{\theta}^{*})\|_{op}, \|\nabla_{\boldsymbol{\theta}}^2 k(\boldsymbol{\theta},\boldsymbol{\theta}^{*})\|_{op}\} \leq \kappa_1^2$ for $\kappa_1 > 0$. Moreover, for all $i, j \in \{1,2,...,d\}, \sup_{\boldsymbol{\theta} \in \mathbb{R}^d} \nabla_{\boldsymbol{\theta}^{*}_i} \nabla_{\boldsymbol{\theta}^{*}_j} \nabla_{\boldsymbol{\theta}_i} \nabla_{\boldsymbol{\theta}_j} k(\boldsymbol{\theta},\boldsymbol{\theta}^{*})|_{\boldsymbol{\theta}^{*}=\boldsymbol{\theta}} \leq \kappa_2^2$ for $\kappa_2 > 0$.
 \end{assumption}

 \begin{assumption}[Decaying kernel derivatives]
 \label{Assumption 3}
 The kernel $k$ is differentiable and admits a $\gamma \in \mathbb{R}$ such that, for all $\boldsymbol{\theta}, \boldsymbol{\theta}^{*} \in \mathbb{R}^d$ satisfying $\|\boldsymbol{\theta}-\boldsymbol{\theta}^{*}\|_2 \geq 1,$
    $$\|\nabla_{\boldsymbol{\theta}} k(\boldsymbol{\theta},\boldsymbol{\theta}^{*}) \|_2 \leq \gamma / \|\boldsymbol{\theta}-\boldsymbol{\theta}^{*}\|_2.$$
\end{assumption}

These assumptions are general and not overly restrictive. Previous studies have commonly adopted the Lipschitz score function \citep{gorham2015measuring, liu2017stein, salim2022convergence}. \cite{shi2024finite} showed that a wide variety of kernels, such as RBF kernels, IMQ kernels, and Gaussian kernels, satisfy Assumptions~\ref{Assumption 2} and \ref{Assumption 3}. In this article, we use the RBF kernels for analyses.

Here, we define the Langevin KSD for theoretical justification. 
Under Assumptions \ref{Assumption 1} and \ref{Assumption 2}, the Langevin KSD is defined as $\text{KSD}_{\Pi}(\mathcal{Q}, \mathcal{V}) \triangleq \sup_{\|\phi\|_{\mathcal{H}_d} \leq 1} \mathbb{E}_{\mathcal{Q}}[\mathcal{A}_{\Pi}\phi] -\mathbb{E}_{\mathcal{V}}[\mathcal{A}_{\Pi}\phi]$. Note that KSD defined in Section~\ref{sec:2} computes the maximum expectation error between an arbitrary distribution $\mathcal{Q}$ and the target, while the Langevin KSD computes the maximum expectation error between two arbitrary distributions $\mathcal{Q}$ and $\mathcal{V}$. According to Lemma 1 in \cite{shi2024finite}, the Langevin KSD is symmetric and satisfies the triangle inequality. We henceforth refer to the Langevin KSD as KSD for convenience. We use $\mathcal{Q}_{\infty}^{(t)}$ to represent the transformed distribution resulting from $t$ iterations of the continuous SVGD with the initial reference distribution $\mathcal{Q}_{\infty}^{(0)}$. Similarly, we let $\mathcal{Q}_{n,m}^{(t)}$ and $\mathcal{Q}_{n}^{(t)}$ denote the transformed distribution produced by $t$ iterations of MC-SVGD and  SVGD with the initial reference distribution $\mathcal{Q}_{n.m}^{(0)}$ and $\mathcal{Q}_{n}^{(0)}$, respectively.

In what follows, we establish the convergence rate of MC-SVGD. A detailed proof is provided in the supplementary material; here, we outline the main ideas. We begin by quantifying the Monte Carlo approximation error of MC-SVGD (Lemma S1), followed by deriving the KSD between $\mathcal{Q}_{n,m}^{(t)}$ and $\mathcal{Q}_{\infty}^{(t)}$ (Theorems S1 and S2). We then present a bound on the KSD between $\mathcal{Q}_{n,m}^{(t)}$ and the target distribution $\Pi$. These results collectively lead to Theorem~\ref{Theorem 3}, which establishes the convergence rate of MC-SVGD. Formal statements and complete proofs can be found in the supplemental Section~A.

\begin{theorem} [A convergence rate of MC-SVGD]
 \label{Theorem 3}
  Let $\Pi$ satisfies Talagrand's $T_1$ inequality from Definition 22.1 in \cite{villani2009optimal}. Suppose Assumptions 1, 2, and 3 hold, fix any $\mathcal{Q}_{\infty}^{(0)}$ which is absolutely continuous with respect to $\Pi$ and $\mathcal{Q}_{n,m}^{(0)}, \mathcal{Q}_n^{(0)}, \mathcal{Q}_{\infty}^{(0)} \in \mathbf{\Pi}_1$. Then the outputs $\mathcal{Q}_{n,m}^{(t)}$ satisfy
    \begin{equation} \label{ksd_cor1_conv}
      \begin{split}
          \min_{0 \leq r \leq t} \text{KSD}_{\Pi}(\mathcal{Q}_{n,m}^{(t)}, \Pi)
    & =\mathcal{O}((\log \log (n))^{-1/2}) + \mathcal{O}(m^{-1/2}),
      \end{split}
  \end{equation}
with probability at least $1-c\delta$ for a universal constant $c>0$.
\end{theorem}


Provided the theorem holds, the KSD between MC-SVGD and the target distribution converges to zero at a rate of $\mathcal{O}((\log \log (n))^{-1/2}) + \mathcal{O}(m^{-1/2})$. We note that achieving the desired convergence rate requires the sum of the step size sequence to be less than or equal to one (supplemental Section~A). Based on our numerical experiments, we recommend using approximately 500 iterations in practice. Accordingly, choosing a step size in the range of 0.0001 to 0.001 satisfies this condition.


\section{Applications}
\label{sec:4}

Here we apply our MC-SVGD method to three different challenging examples: (1) a Potts model, (2) a Conway--Maxwell--Poisson regression model, and (3) an exponential random graph model. The performance of our MC-SVGD is assessed by comparing its posterior density estimates to those of the gold standard method. In cases where perfect sampling is feasible, the exchange algorithm \citep{murray2006} serves as the gold standard; otherwise, the double Metropolis-Hastings algorithm \citep{liang2010double} is employed. 
The code for the examples is implemented in {\tt R} and {\tt C++}, using \texttt{Rcpp} and \texttt{RcppArmadillo} packages \citep{eddelbuettel2011rcpp}. 
The highest posterior density (HPD) intervals are calculated using an {\tt R} package \texttt{coda} \citep{plummer2006}. All codes are run on dual 16-core AMD Ryzen 9 7950X processors. 
Particles are updated in parallel using 32 cores. Derivations for gradients of each application are provided in the supplemental Section~C. The source code can be downloaded from \url{https://github.com/codinheesang/MCSVGD}.

\subsection{A Potts Model}

The Potts model \citep{potts1952some}, a generalized version of the Ising model \citep{ising1924beitrag, lenz1920beitrag}, can be used to model discrete-valued spatial data on a lattice. For an observed $N \times N$ lattice $\textbf{x} = \left\{x_{i}\right\}$ with $x_{i} \in \left\{1,\cdots,K\right\}$, the likelihood function of the Potts model is given by 
$$L(\theta | \textbf{x}) = \frac{1}{Z(\theta)} \exp\left\{\theta S(\textbf{x})\right\},$$
$$S(\textbf{x})=\sum_{i \sim j} I(x_i = x_j),$$
where $\theta >0$, $i \sim j$ indicates neighboring elements, and $I(\cdot)$ denotes the indicator function. A large value of $\theta$ provides a high expected number of neighboring pairs having the same value. Calculation of the normalizing function $Z(\theta)$ is computationally expensive because it requires summation over all $K^{N \times N}$ possible outcomes for the model. 

We study the Antarctic ice sheet dataset \citep{pollard2015large}, which provides the observed thickness of ice for a $171 \times 171$ lattice with 20-km resolution at the center of the South Pole. The observed thickness is classified into four different categories as follows: $x_i$ = 0 for no ice; $x_i$ = 1 for $0< \text{thickness} \leq 1,000$, $x_i$ = 2 for  $1,000< \text{thickness} \leq 2,000$; $x_i$ = 3 for  $ \text{thickness} > 2,000$. We consider our MC-SVGD method and double Metropolis-Hastings algorithm (DMH) for carrying out inference for $\theta$. For comparison of our method to DMH, we examine the posterior inference results for $\theta$ and the resulting estimated surfaces of ice thickness. 

For MC-SVGD, we estimate $\theta^{\text{MAP}}$ by a preliminary run of the naive MC-SVGD with a single particle over 300 iterations.
An initial set of 64 particles is an i.i.d. sample from a normal distribution with a mean equal to the MAP estimate and variance derived from a Poisson regression of the response onto the number of neighboring pairs having the same value.
Our MC-SVGD was run for $500$ iterations with $m$ = 50, threshold value = $m/3$, and $\epsilon$ = 0.0001. 
For generating MC samples, we run 80 cycles of the Swendsen-Wang algorithm \citep{swendsen1987nonuniversal,geyer2022potts}, with each cycle updating the entire $171 \times 171$ lattice, and retain the last 50 cycles, which provides 50 distinct lattice samples.
We run DMH with 30 cycles of (inner) Swendsen-Wang updates for 11,000 iterations until convergence and discard the first 1,000 samples as burn-in.

\newcolumntype{M}[1]{>{\centering\arraybackslash}m{#1}}
\newcolumntype{P}[1]{>{\centering\arraybackslash}p{#1}}
\begin{table}[t]
  \centering
    \begin{tabular}{lrrr}
        \toprule
         Method & Posterior mean & 95\% HPD interval & Time (min) \\
        \midrule
         DMH & 1.23 & (1.22, 1.24) & 34.00\\
         MC-SVGD & 1.23 & (1.22, 1.24) & 2.13 \\
        \bottomrule
    \end{tabular}
    \caption{Posterior inference results of $\theta$ for the Potts model on the Antarctic ice sheet dataset. 
}
    \label{tab:potts}
\end{table}

Table \ref{tab:potts} presents the posterior summary statistics for $\theta$ along with the computation times for DMH and MC-SVGD. The results indicate that both methods yield identical estimated posterior means and 95\% HPD intervals. However, the computation time for MC-SVGD was significantly shorter at 2.13 minutes, compared to 34 minutes for DMH. 
Figure \ref{fig:potts_icesheet} shows the observed ice sheet surface alongside the estimated surface derived from the MC-SVGD sample, with the estimated surface closely matching the observed data. In conclusion, the proposed MC-SVGD performs similarly to DMH---recognized as one of the most effective algorithms for addressing intractable normalizing function problems---yet requires only a fraction of the computational time.

\begin{figure}[t]
     \centering
     \begin{subfigure}[b]{0.45\textwidth}
         \centering
         \includegraphics[width=\textwidth]{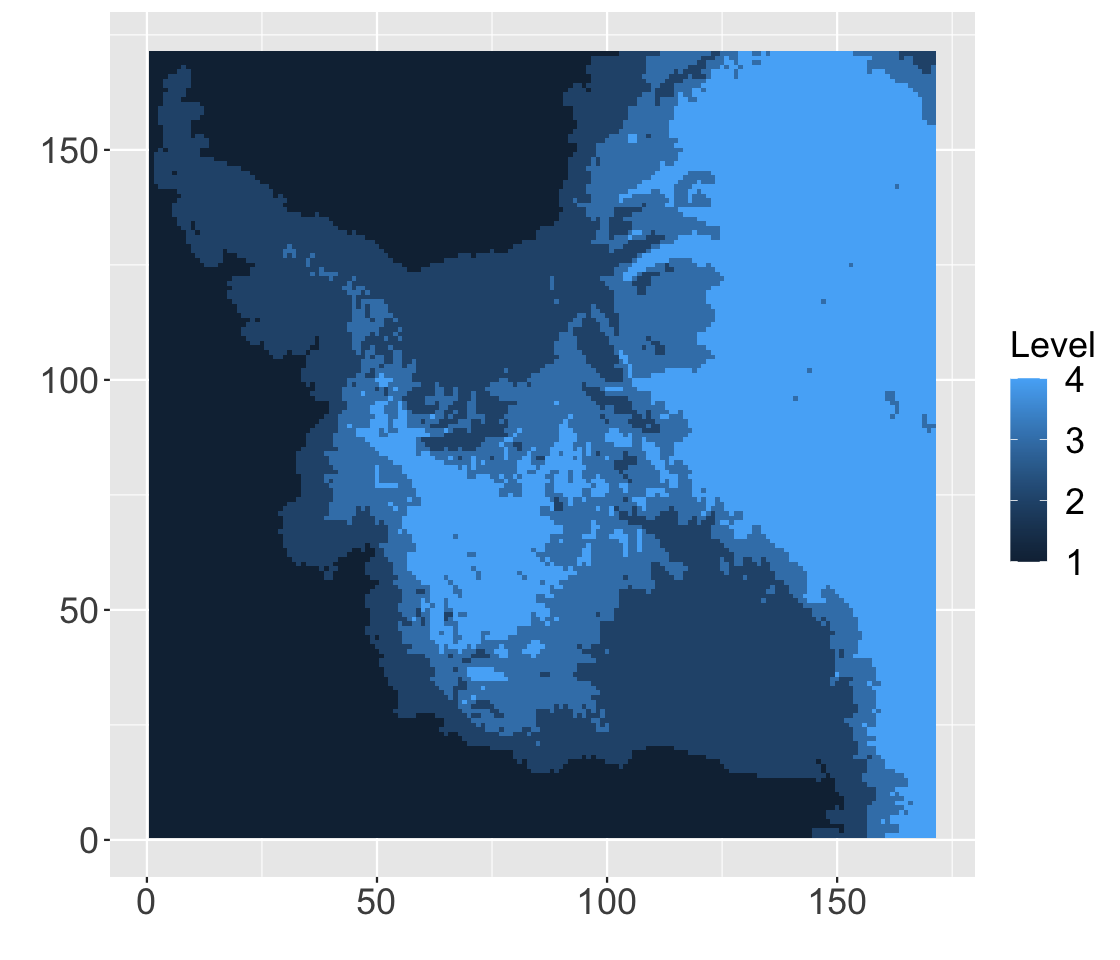}
         \caption{True surface}
         \label{fig:potts_true}
     \end{subfigure}
     \hfill
     \begin{subfigure}[b]{0.45\textwidth}
         \centering
        \medskip         
         \includegraphics[width=\textwidth]{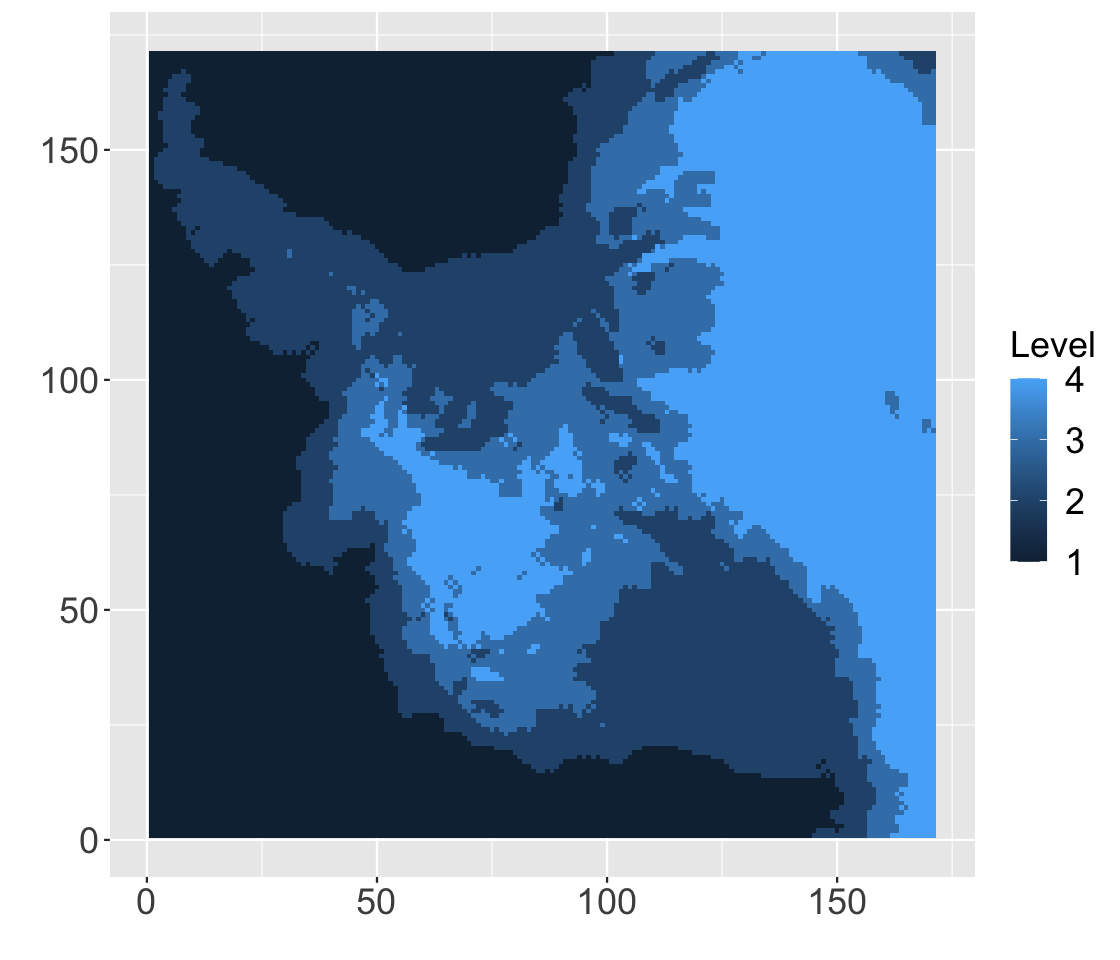}
         \caption{Estimated surface}
         \label{fig:potts_estimated}
     \end{subfigure}
        \caption{Observed ice sheet surface and estimated surface via MC-SVGD for the Potts example.}
        \label{fig:potts_icesheet}
\end{figure}

\subsection{A Conway--Maxwell--Poisson Regression Model}

A Conway--Maxwell--Poisson (COMP) distribution \citep{conway1962network} is a generalized version of the Poisson distribution, which can account for underdispersion (i.e., mean smaller than variance) and overdispersion (i.e., mean larger than variance). 
The probability mass function of a count variable $Y \sim \text{COMP}(\eta, \nu)$ is given by 
$$P(Y=y) = \frac{1}{Z(\eta, \nu)} \left( \frac{\eta^y}{y!} \right)^\nu,$$
where $\eta > 0$ approximates the mode of the distribution, and $\nu \geq 0$ is a dispersion parameter for which $\nu > 1$, $\nu = 1$, and $0\leq \nu < 1$ represent under-, equi-, and overdispersion, respectively.
The normalizing function $Z(\eta,\nu)= \sum_{j=0}^\infty \left( \frac{\eta^j}{j!} \right)^\nu$ is an infinite sum, which makes the function intractable.

\begin{table}[b]
  \centering
    \begin{tabular}{llrrr}
        \toprule
         N & Method & Posterior mean & 95\% HPD interval & Time (sec) \\
        \midrule
         \multirow{2}{*}{225} & Exchange & 1.05 & (0.87, 1.25) & 30.02\\
                             & MC-SVGD & 1.05 & (0.85, 1.21) & 5.08\\
         \midrule
         \multirow{2}{*}{400} & Exchange & 0.93 & (0.78, 1.07) & 50.30\\
                             & MC-SVGD & 0.93 & (0.77, 1.05) & 7.37\\
        \midrule
        \multirow{2}{*}{2500} & Exchange & 1.01 & (0.95, 1.07) & 256.49\\
                            & MC-SVGD & 1.01 & (0.96, 1.07) & 38.56\\
        \bottomrule
    \end{tabular}
    \caption{Posterior inference results of $\beta_1$ for the COMP regression model on the simulated data with increasing size $N$ of data. The simulated truth of $\beta_1$ is 1. 
    }
    \label{tab:comp_theta1}
\end{table}

For an observed response $y_i$ and the covariates $\textbf{x}_i = ( x_{i1}, x_{i2},\dots, x_{i,p})^\top$, we consider the following COMP regression model: $y_i \sim \text{COMP}(\eta_i, \nu)$ and $\log(\eta_i) = \sum_{j=1}^p \beta_j x_{ij}$ for $i = 1, \dots, N$. The likelihood function of the COMP regression model is given by
\begin{align*}
    L(\beta_1, \cdots, \beta_p, \nu \mid \textbf{y}) &= \prod_{i=1}^N \frac{1}{Z(\eta_i, \nu)} \left( \frac{\eta_i^{y_i}}{y_i!}\right)^\nu\\
    &= \frac{\exp \left\{ \nu \sum_{j=1}^{p} \beta_j S_j(\textbf{y}) - \nu S_{p+1}(\textbf{y})\right\}}{\prod_{i=1}^N Z(\eta_i, \nu)},
\end{align*}
where $S_j(\textbf{y})=\sum_{i=1}^N x_{ij}y_i$, $j = 1,\dots,p$  and $S_{p+1}(\textbf{y})=\sum_{i=1}^N \log(y_i!)$ are sufficient statistics. From the COMP regression model, we simulated three different datasets with increasing data size $N \in \{225, 400, 2500\}$ exactly from the model through a fast rejection sampling \citep{chanialidis2018efficient}. We set $p = 3$, $\beta_1 = \beta_2 = 1, \beta_3 = 0.1$, and $\nu = \exp(0.5)$. The $\nu$ is assumed to be known and fixed at truth. In this example we consider our MC-SVGD method and the exchange algorithm.
The exchange algorithm is asymptotically exact, i.e., it generates a sequence
whose asymptotic distribution is exactly equal to the target posterior distribution. We compare the posterior inference results for $\boldsymbol{\theta} = (\beta_1, \beta_2, \beta_3)^\top$ and the computing times with respect to increasing simulated data size.

For MC-SVGD, we obtained $\boldsymbol{\theta}^{\text{MAP}}$ through a preliminary run of the naive MC-SVGD using a single particle over 300 iterations.
An initial set of 96 particles is drawn from a multivariate normal distribution $N(\boldsymbol{\theta}^{\text{MAP}},\boldsymbol{\Sigma})$, where $\boldsymbol{\Sigma}$ is the estimated covariance matrix of $\boldsymbol{\theta}$ derived from a Poisson regression of the response onto the covariates. 
MC-SVGD was run for $500$ iterations with $m$ = 50, threshold value = $m/3$, and $\epsilon$ = 0.0005 for $N \in \{225,400\}$ and $\epsilon$ = 0.0001 for $N=2500$. The $m = 50$ samples for MC approximation are generated exactly from the data model through the rejection sampling.
The exchange algorithm was run for 51,000 iterations to achieve convergence, with the initial 1,000 samples discarded as burn-in.

\begin{figure}[t]
     \centering
     \begin{figure}
         \centering
         \includegraphics[width=.5\textwidth]{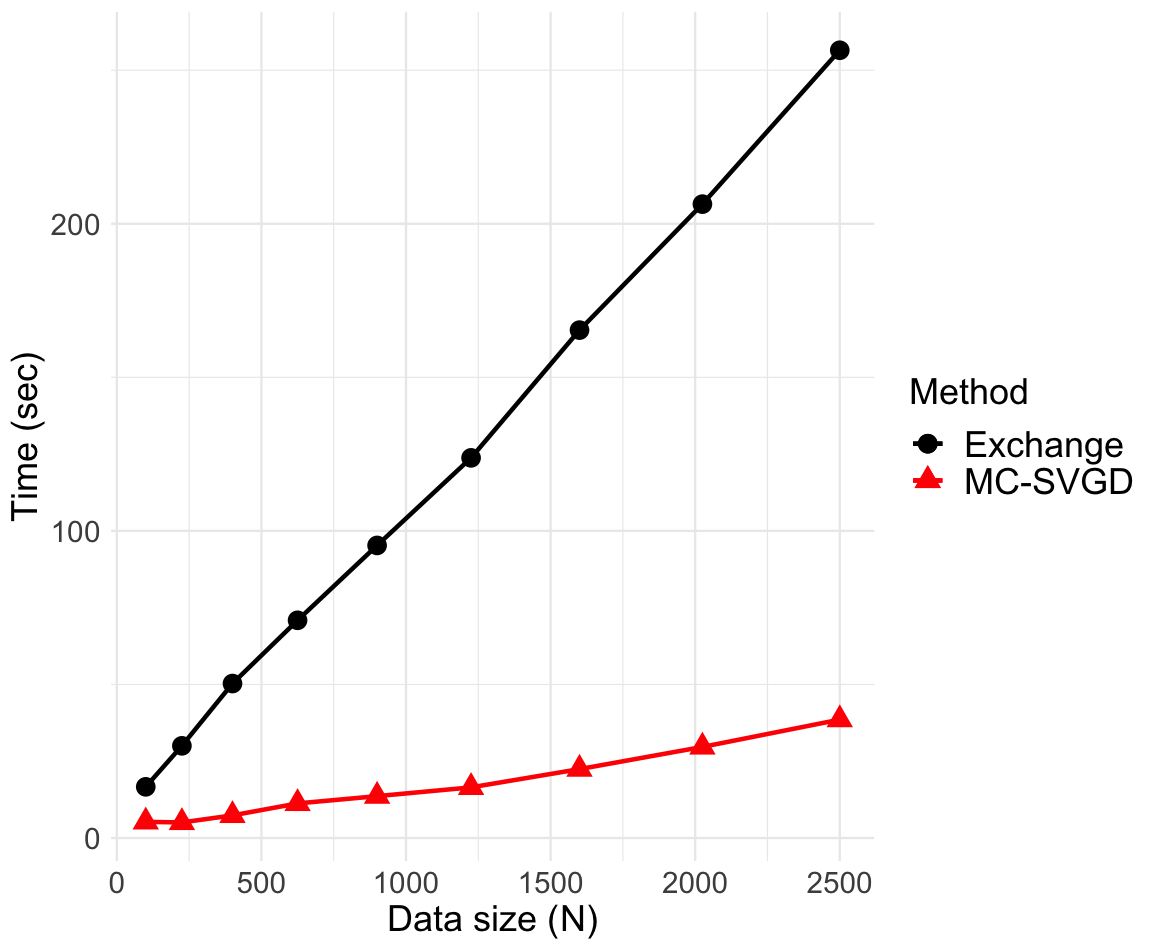}
     \end{figure}
        \caption{Comparison of computing time between the exchange and MC-SVGD algorithms with increasing data size $N$ for the COMP example.}
        \label{fig:comp_time}
\end{figure}

Table \ref{tab:comp_theta1} presents posterior summary statistics of $\beta_1$ and computing times for the exchange and MC-SVGD algorithms with increasing data size $N$.
The results demonstrate that MC-SVGD achieves estimated posterior means and 95\% HPD intervals similar to those of the asymptotically exact exchange algorithm, while being at least 5.9 times faster.
Comparable results for the other parameters are provided in the supplemental Section~D.
Figure \ref{fig:comp_time} shows the computing time for the exchange and MC-SVGD algorithms as data size increases. The results indicate that the computational advantage of MC-SVGD over the exchange method becomes more pronounced as the data size $N$ grows. Overall, our algorithm delivers reliable inference and demonstrates significant computational benefits compared to the exchange algorithm, especially for large datasets.

\subsection{An Exponential Random Graph Model}

An exponential random graph model (ERGM) \citep{robins2007introduction,hunter2008ergm} describes the relationship among nodes of a network. Let $\textbf{x}$ be an $N \times N$ adjacency matrix of an undirected network with $N$ nodes.
For all $i \neq j$, $x_{i,j}=1$ if nodes $i$ and $j$ are connected and $x_{i,j}=0$ otherwise. For $i = 1,\dots,N$, $x_{i,i} = 0$. We study the Faux Mesa high school network dataset \citep{resnick1997protecting}, which demonstrates an in-school friendship network among $N = 205$ students. The vertex attributes we consider are `Grade' (7--12) and `Sex' (male or female). Consider an undirected ERGM whose likelihood function is given by 
\begin{align*}
    L(\boldsymbol{\theta} \mid \textbf{x}) &= \frac{1}{Z(\boldsymbol{\theta})} \exp \left\{\boldsymbol{\theta}^\top \textbf{S}(\textbf{x})\right\}\\
    S_1(\textbf{x}) &= \sum_{i=1}^N \binom{x_{i+}}{1}, \\
    S_{g-5}(\textbf{x}) &= \sum_{i<j} x_{i,j} (1\{ \text{grade}_i = g\} \times 1\{ \text{grade}_j = g\}), \: g = 7,\cdots,12, \\
    S_8(\textbf{x}) &=\sum_{i<j} x_{i,j} (1\{ \text{sex}_i = \text{sex}_j\}), \\
    S_9(\textbf{x}) &= e^{\tau_d} \sum_{k=1}^{N-1} \left\{1-(1-e^{-\tau_d})^k\right\} \text{D}_k(\textbf{x}), \\
    S_{10}(\textbf{x}) &= e^{\tau_s} \sum_{k=1}^{N-2} \left\{1-(1-e^{-\tau_s})^k\right\} \text{ESP}_k(\textbf{x}),
\end{align*}
where $\boldsymbol{\theta} = (\theta_1, \dots, \theta_{10})^\top$ is the collection of the model parameters, $S_1(\textbf{x})$ is the total number of edges, $S_{g-5}(\textbf{x})$ is the number of edges between two nodes with the same grade $g$ for $g = 7,\dots,12$,
$S_8(\textbf{x})$ is the number of edges between two nodes with the same sex, $S_9(\textbf{x})$ is the geometric weighted degree (GWD) where $\text{D}_k(\textbf{x})$ counts the number of nodes that have $k$ neighbors, and
$S_{10}(\textbf{x})$ is the geometrically weighted edge-wise shared partnership (GWESP) statistic where $\text{ESP}_k(\textbf{x})$ denotes the number of edges between two nodes that share exactly $k$ neighbors. 
The $\tau_d$ and $\tau_s$ are assumed to be known and fixed at 0.25. In this example we consider our MC-SVGD, DMH, and a stochastic variational inference (SVI) method using self-normalized importance sampling (see Section 5.2 of \citet{tan2020bayesian} for further details).

For MC-SVGD, we estimate $\boldsymbol{\theta}^{\text{MAP}}$ by a preliminary run of the naive MC-SVGD with a single particle over $500$ iterations. 
An initial set of 320 particles is sampled from a multivariate normal distribution $N(\boldsymbol{\theta}^{\text{MAP}},\boldsymbol{\Sigma})$, where $\boldsymbol{\Sigma}$ is set to the covariance matrix estimate of $\boldsymbol{\theta}$ obtained by MPLE.
MC-SVGD was run for $500$ iterations with $m$ = 50, threshold value = $m/1.5$, and $\epsilon$ = 0.0005.
To generate MC samples, we run 10 cycles of a random scan Gibbs sampler \citep{park2020function}, with each cycle updating the entire $205 \times 205$ network, discard the first 10 cycles as burn-in, and then retain every 500-th sample, which yields 50 unique networks. 
The DMH algorithm with 10 cycles of (inner) random scan Gibbs updates was run for 81,000 iterations until convergence, and the first 1,000 samples were discarded as burn-in.

The framework of the SVI algorithm resembles our MC-SVGD, but it limits the variational distribution class for the posterior to the Gaussian distribution. The SVI method iteratively updates the mean and covariance matrix of the Gaussian distribution by maximizing the evidence lower bound (ELBO) via stochastic gradient ascent until convergence. 
Similar to our method, the ELBO gradients are estimated by MC approximation if the ESS, based on SNIS weights, is below a given threshold; otherwise, SNIS is used.
Following \citet{tan2020bayesian}, we use the MC sample size $m$ = 50 and set the threshold value to $m/1.5$. The MC samples are generated from the data model via the Metropolis-Hastings (MH) algorithm using {\tt{R}} package {\texttt{ergm}}. After discarding the first 30,000 iterations as burn-in, we retain every $1,024$-th sample to produce 50 samples. The SVI algorithm converged in $7,000$ iterations.

\begin{table}[t]
  \centering
    \begin{tabular}{lrrr}
        \toprule
         Method & Posterior mean & 95\% HPD interval & Time (min) \\
        \midrule
        DMH & 1.49 & (1.23, 1.75) & 71.85 \\
        SVI & 1.53 & (1.25, 1.81) & - \\
        MC-SVGD & 1.42 & (1.22, 1.66) & 4.47 \\
        \bottomrule
    \end{tabular}
    \caption{Posterior inference results of $\theta_{10}$ (GWESP) for the ERGM model on the Faux Mesa high school network dataset.
    }
    \label{tab:ergm_theta10}
\end{table}

Table \ref{tab:ergm_theta10} provides posterior summary statistics of $\theta_{10}$ (GWESP) and computing times for the three methods. Results for other parameters are provided in the supplemental Section~D.
The DMH and MC-SVGD methods yield similar estimates for the estimated posterior mean and 95\% HPD interval. SVI provides similar values but slightly over- or underestimates some of the parameters, including $\theta_{10}$. The DMH and MC-SVGD algorithms were implemented in {\tt C++} using {\texttt Rcpp}. The code for the SVI algorithm, originally written in {\tt Julia} by \citet{tan2020bayesian}, encountered compatibility issues, so we executed it in {\tt R}.
Thus, it is inequitable to compare the computing time of SVI with that of the other two methods. However, we anticipate that SVI takes as long as MC-SVGD due to their similarities. Our MC-SVGD method took only 4.47 minutes, whereas DMH took 71.85 minutes. In summary, our MC-SVGD offers multiple advantages, including ease of tuning and a substantial reduction in computational cost, while preserving inferential accuracy comparable to other methods.

\section{Discussion}
\label{sec:5}

In this article, we propose a new Stein variational gradient descent algorithm that achieves reliable inference for doubly intractable target distributions with a considerable reduction in computing time through efficient SNIS estimation and parallel computation for gradient calculations. We have also established the convergence rate of our MC-SVGD method under mild conditions by analyzing the errors associated with finite particle estimation and Monte Carlo approximation.
We have studied the application of MC-SVGD to several challenging simulated and real data examples. This shows that MC-SVGD enables accurate Bayesian inference with much lower computational expense compared to existing methods.

Note that MC-SVGD proves practical for examples with a moderate parameter dimension. However, its application to higher dimensional problems remains an open challenge. For example, Ising network models used in network psychometrics can have $d \geq 1,000$ to describe pairwise interaction among survey questions \citep{van2014new, park2022bayesian}. In this case, applying MC-SVGD directly poses challenges, as it necessitates a substantial increase in the number of particles to adequately represent the critical regions of $\mathbf{\Theta}$. Future research could explore the incorporation of advanced projection techniques suggested in recent SVGD literature \citep{chen2019projected, chen2020projected, wang2022projected}. These methods primarily aim to reduce dimensionality by projecting parameters into a lower-dimensional subspace, leveraging gradient information. Improving the quality of IS estimates, albeit at a higher computational cost, is another potential research direction. Possible approaches include annealed IS \citep{neal2001annealed}, dynamically weighted average IS \citep{liang2002dynamically}, and adaptive multiple IS \citep{cornuet2012adaptive}. A detailed exploration of these methods could potentially enhance the efficiency of MC-SVGD.


\section*{Acknowledgements}
The authors are grateful to anonymous reviewers for their careful reading and valuable comments. This work was supported by the National Research Foundation of Korea (RS-2025-00513129). 

\section*{Data Availability Statement}
The datasets and source code used in this study are available upon request.

\section*{Disclosure Statement}
The authors have no conflicts of interest to disclose.




\section*{Supplementary Material}
Supplementary materials available online contain theoretical results, sensitivity analysis of the algorithm, derivations for gradient, and extra results.

\clearpage
\appendix
\begin{center}
\title{\LARGE\bf Supplementary Material for ``A Stein Gradient Descent Approach for Doubly Intractable Distributions''}\\~\\
\author{\Large{Heesang Lee, Songhee Kim, Bokgyeong Kang, and Jaewoo Park}}
\end{center}

\newtheorem{stheorem}{Theorem}
\newtheorem{slemma}{Lemma} 
\renewcommand{\thestheorem}{S\arabic{stheorem}}
\renewcommand{\theslemma}{S\arabic{slemma}}

The supplementary material includes theoretical results, along with sensitivity analysis results offering practical guidelines for tuning its components. Additionally, it provides detailed derivations for the gradient and presents further results related to the models discussed in the application section.

\section{Theoretical Results for MC-SVGD}
\label{appendix_a}

In this section, we present the technical proofs of the main results. We first define $\Gamma(\mathcal{Q}, \mathcal{V})$ as the set of all couplings of $\mathcal{Q}, \mathcal{V} \in \mathbf{\Pi}_1$. 
Let $W_1(\mathcal{Q}, \mathcal{V}) \triangleq \inf_{\mathcal{G} \in \Gamma(\mathcal{Q}, \mathcal{V})} \mathbb{E}_{\mathcal{G}} [\|\boldsymbol{\theta}^{*} - \boldsymbol{\theta}\|_2]$ be the 1-Wasserstein distance between $\mathcal{Q}$ and $\mathcal{V}$, where $\boldsymbol{\theta}$ and $\boldsymbol{\theta}^{*}$ are drawn independently from $\mathcal{Q}$ and $\mathcal{V}$ respectively. We define $m_{\mathcal{Q}, \boldsymbol{\theta}^{*}} \triangleq \mathbb{E}_{\mathcal{Q}}[\|\cdot - \boldsymbol{\theta}^{*}\|]_2$ for each $\boldsymbol{\theta}^{*} \in \mathbb{R}^d$. Let $m_{\mathcal{Q}, \Pi} \triangleq \mathbb{E}_{\mathcal{Q} \otimes \Pi} [\|\boldsymbol{\theta}-\boldsymbol{\theta}^{*}\|_2]$ and $M_{\mathcal{Q}, \Pi} \triangleq \mathbb{E}_{\mathcal{Q} \otimes \Pi} [\|\boldsymbol{\theta}-\boldsymbol{\theta}^{*}\|^2_2]$, where $\mathcal{Q} \otimes \Pi$ means the independent coupling. The remaining notation follows the definitions provided in the main manuscript.

Lemma 13 of \cite{gorham2020stochastic} provides an upper bound for the 1-Wasserstein distance between SVGD and stochastic SVGD which approximates probability distribution using mini-batch updates. Extending their results, Lemma~\ref{onestep_lemma} establishes an upper bound of the 1-Wasserstein distance between the pushforward distribution of MC-SVGD and that of SVGD. 
\begin{slemma} [MC approximation error of MC-SVGD]
    \label{onestep_lemma}
    Suppose the MC samples $\lbrace\textbf{y}_{k}\rbrace_{k=1}^m$ are generated from a polynomial ergodic Markov chain whose stationary distribution is $f(\cdot|\boldsymbol{\theta})$. 
    Let $\mathcal{Q}=\sum_{i=1}^n \delta_{\boldsymbol{\theta}_i}/n$ be the sampling distribution of particles $\{\boldsymbol{\theta}_i\}_{i=1}^n$.
    Suppose $\nabla \log \pi(\cdot)k(\cdot, \boldsymbol{\theta}^{*})$ is continuous for each $\boldsymbol{\theta}^{*} \in \mathbb{R}^d$ and let
    $$ f_0(\boldsymbol{\theta}) \triangleq \sup_{\boldsymbol{\theta}^{*} \in \mathbb{R}^d} \|\nabla \log \pi(\boldsymbol{\theta})\|_\infty |k(\boldsymbol{\theta}, \boldsymbol{\theta}^{*})|,$$
    $$ f_1(\boldsymbol{\theta}) \triangleq \sup_{\boldsymbol{\theta}^{*} \in \mathbb{R}^d} \|\nabla_{\boldsymbol{\theta}}( \nabla \log \pi(\boldsymbol{\theta})k(\boldsymbol{\theta}, \boldsymbol{\theta}^{*}))\|_{op}.$$ 
If $f_0$ is uniformly integrable, and $f_0$ and $f_1$ are bounded on each compact set, then, for any $\epsilon > 0$, $$W_1(\Phi_{n,m}(\mathcal{Q}), \Phi_n(\mathcal{Q})) \leq \mathcal{O}(m^{-1/2}).$$
\end{slemma}

\begin{proof}
By Kantorovich-Rubinstein duality, we have 
    \begin{align*}
        W_1({\Phi}_{n,m}(\mathcal{Q}), \Phi_n(\mathcal{Q}))&=\sup_{\|f\|_L\leq 1} {\Phi}_{n,m}(\mathcal{Q})(f)-\Phi_n(\mathcal{Q})(f)\\
  &=\sup_{\|f\|_L\leq 1} \frac{1}{n} \sum_{i=1}^n f({\mathcal{T}}_{n,m}(\boldsymbol{\theta}_i)) - f(\mathcal{T}_{n}(\boldsymbol{\theta}_i))\\
  &\leq \frac{1}{n} \sum_{i=1}^n \|{\mathcal{T}}_{n,m}(\boldsymbol{\theta}_i) - \mathcal{T}_{n}(\boldsymbol{\theta}_i)\|_2,
    \end{align*}
where $\|f\|_L$ denotes the 1-Lipschitzs function. By applying the ergodic theorem, we obtain 
\begin{align*}
    \Vert {\mathcal{T}} _{n,m}(\boldsymbol{\theta}_i) - \mathcal{T}_{n}(\boldsymbol{\theta}_i) \Vert^2_2 &= \sum_{\ell=1}^p \bigg\vert \frac{\epsilon}{n} \sum_{j=1}^n k^\ell(\boldsymbol{\theta}_j, \boldsymbol{\theta}) ( \nabla \log \widehat{\pi}_\ell(\boldsymbol{\theta}_j)- \nabla \log \pi_\ell(\boldsymbol{\theta}_j)) \bigg\vert^2\\
    & = \sum_{\ell=1}^p \epsilon^2 \bigg\vert \frac{1}{n}\sum_{j=1}^n k^\ell(\boldsymbol{\theta}_j, \boldsymbol{\theta}) ( \nabla \log \widehat{\pi}_\ell(\boldsymbol{\theta}_j)- \nabla \log \pi_\ell(\boldsymbol{\theta}_j)) \bigg\vert^2,
\end{align*}
where
\begin{align*}
    &\frac{1}{n}\sum_{j=1}^n k^\ell(\boldsymbol{\theta}_j, \boldsymbol{\theta}) ( \nabla \log \widehat{\pi}_\ell(\boldsymbol{\theta}_j)- \nabla \log \pi_\ell(\boldsymbol{\theta}_j))\\
    &\leq \frac{1}{n} \sum_{j=1}^n k^\ell(\boldsymbol{\theta}_j, \boldsymbol{\theta}) \vert \nabla \log \widehat{\pi}_\ell(\boldsymbol{\theta}_j)- \nabla \log \pi_\ell(\boldsymbol{\theta}_j)\vert = \mathcal{O}(m^{-1/2}).
\end{align*}
Therefore, we have
$$W_1({\Phi}_{n,m}(\mathcal{Q}), \Phi_n(\mathcal{Q})) \leq \mathcal{O}(m^{-1/2}).$$
    
\end{proof}

Compared to the traditional SVGD, MC-SVGD introduces additional MC error for approximating the gradient of the log posterior. 
Lemma~\ref{onestep_lemma} quantifies this MC approximation error in the pushforward distribution. 
Lemma 2 in \cite{shi2024finite} demonstrates that, given the score function $\nabla \log \pi$ and kernel $k$ satisfy a commonly used pseudo-Lipschitz condition, the pushforward distribution $\Phi$ is pseudo-Lipschitz for 1-Wasserstein distance, thereby and quantifying the finite particle approximation error of SVGD.
Based on Lemma~\ref{onestep_lemma} and Lemma 2 of \cite{shi2024finite}, Theorem~\ref{Theorem 1} quantifies the 1-Wasserstein distance between MC-SVGD and continuous SVGD.
\begin{stheorem} [1-Wasserstein distance between MC-SVGD and continuous SVGD]
 \label{Theorem 1} Suppose Assumptions 1, 2, and 3 hold. For any $\mathcal{Q}^{(0)}_{n,m}, \mathcal{Q}^{(0)}_n, \mathcal{Q}^{(0)}_\infty \in \mathbf{\Pi}_1$, the transformed distributions at iteration $t$, $\mathcal{Q}_{n,m}^{(t)}$, $\mathcal{Q}_n^{(t)}$, and $\mathcal{Q}_\infty^{(t)}$, satisfy
  \begin{align*}
  &W_1(\mathcal{Q}^{(t)}_{n,m}, \mathcal{Q}^{(t)}_\infty)\\
  &\leq (W_1(\mathcal{Q}_{n,m}^{(0)}, \mathcal{Q}_n^{(0)})+W_1(\mathcal{Q}_n^{(0)}, \mathcal{Q}_\infty^{(0)}))  \exp {(b^{(t-1)} (A+B \exp{(Cb^{(t-1)})}))} + \mathcal{O}(m^{-1/2}),  
\end{align*}
  where $b^{(t-1)} \triangleq \sum_{r=0}^{t-1} \epsilon^{(r)}$, $A=(c_1+c_2)(1+m_{\Pi,\boldsymbol{\theta}})$, $B=c_1 (m_{\mathcal{Q}_{n,m}^{(0)}, \Pi}+m_{\mathcal{Q}_{n}^{(0)}, \Pi}) + c_2 (m_{\mathcal{Q}_n^{(0)}, \Pi}+m_{\mathcal{Q}_{\infty}^{(0)}, \Pi})$, $C=\kappa_1^2(3L+1)$, $c_1=\max\{\kappa_1^2L, \kappa_1^2\}$, and $c_2=\kappa_1^2(L+1)+L\max\{\gamma, \kappa_1^2\}$.
\end{stheorem}

\begin{proof}
By triangle inequality, we have
 \begin{align}
   &W_1(\mathcal{Q}^{(t+1)}_{n,m}, \mathcal{Q}^{(t+1)}_\infty) \nonumber \\
   &\leq W_1(\mathcal{Q}^{(t+1)}_{n,m}, \mathcal{Q}^{(t+1)}_n) + W_1(\mathcal{Q}^{(t+1)}_{n}, \mathcal{Q}^{(t+1)}_\infty) \nonumber\\
   &=W_1({\Phi}_{n,m}(\mathcal{Q}_{n,m}^{(t)}), \Phi_n(\mathcal{Q}_{n}^{(t)}))+W_1(\mathcal{Q}^{(t+1)}_{n}, \mathcal{Q}^{(t+1)}_\infty) \nonumber\\
   & \leq W_1({\Phi}_{n,m}(\mathcal{Q}_{n,m}^{(t)}), \Phi_n(\mathcal{Q}_{n,m}^{(t)})) + W_1(\Phi_n(\mathcal{Q}_{n,m}^{(t)}), \Phi_n(\mathcal{Q}_n^{(t)})) + W_1(\mathcal{Q}^{(t+1)}_{n}, \mathcal{Q}^{(t+1)}_\infty) .  \label{eq:th1:bound}
 \end{align}
In what follows, we will demonstrate that each component of \eqref{eq:th1:bound} is bounded.

First, according to Lemma~\ref{onestep_lemma}, we have 
\begin{equation} \label{thm1pf_first}
  W_1({\Phi}_{n,m}(\mathcal{Q}_{n,m}^{(t)}), \Phi_n(\mathcal{Q}_{n,m}^{(t)})) \leq \mathcal{O}(m^{-1/2}).  
\end{equation}

Second, by Lemma 2 of \cite{shi2024finite} and the inequality $(1+x) \leq e^x$, we have
 \begin{align}
   &W_1(\Phi_n(\mathcal{Q}_{n,m}^{(t)}), \Phi_n(\mathcal{Q}_n^{(t)}))\nonumber\\
   &\leq (1+\epsilon^{(t)} D^{(t)})W_1(\mathcal{Q}_{n,m}^{(t)}, \mathcal{Q}_n^{(t)})\nonumber\\
   & \leq W_1(\mathcal{Q}_{n,m}^{(0)}, \mathcal{Q}_n^{(0)}) \prod_{r=0}^t (1+ \epsilon^{(r)} (c_1(1+m_{\mathcal{Q}_{n,m}^{(r)}, \boldsymbol{\theta}^*}) +c_2(1+m_{\mathcal{Q}_n^{(r)}, \boldsymbol{\theta}^*})))\nonumber\\
   &\leq W_1(\mathcal{Q}_{n,m}^{(0)}, \mathcal{Q}_n^{(0)}) \exp \left(\sum_{r=0}^t \epsilon^{(r)} (c_1(1+m_{\mathcal{Q}_{n,m}^{(r)}, \boldsymbol{\theta}^*}) +c_2(1+m_{\mathcal{Q}_n^{(r)}, \boldsymbol{\theta}^*}))\right) \label{eq:thm1:com2:1}.
\end{align}
According to Lemma 3 of \cite{shi2024finite}, which illustrates the growth of the moment, we have
$$c_1(1+m_{\mathcal{Q}_{n,m}^{(r+1)}, \boldsymbol{\theta}^*}) +c_2(1+m_{\mathcal{Q}_n^{(r+1)}, \boldsymbol{\theta}^*}) \leq A + (c_1 m_{\mathcal{Q}_{n,m}^{(0)}, \Pi} + c_2 m_{\mathcal{Q}_n^{(0)}, \Pi}) \exp{(Cb^{(r)})},$$
where $A=(c_1+c_2)(1+m_{\Pi, \boldsymbol{\theta}})$ and $C=\kappa_1^2(3L+d)$. Thus, we obtain
\begin{align}
   &\sum_{r=0}^t \epsilon^{(r)} (c_1(1+m_{\mathcal{Q}_{n,m}^{(r)}, \boldsymbol{\theta}^*}) +c_2(1+m_{\mathcal{Q}_n^{(r)}, \boldsymbol{\theta}^*}))\nonumber\\
   &\leq \max_{0\leq r \leq t} (c_1(1+m_{\mathcal{Q}_{n,m}^{(r)}, \boldsymbol{\theta}^*}) +c_2(1+m_{\mathcal{Q}_n^{(r)}, \boldsymbol{\theta}^*})) \sum_{r=0}^t \epsilon^{(r)}\nonumber\\
   &\leq b^{(r)} (A+(c_1 m_{\mathcal{Q}_{n,m}^{(0)}, \Pi} + c_2 m_{\mathcal{Q}_n^{(0)}, \Pi}) \exp{(Cb^{(t-1)})})\nonumber\\
   &\leq b^{(t)} (A+(c_1 m_{\mathcal{Q}_{n,m}^{(0)}, \Pi} + c_2 m_{\mathcal{Q}_n^{(0)}, \Pi})\exp{(Cb^{(t)})}) \label{eq:thm1:com2:2}. 
\end{align}
Substituting the result of \eqref{eq:thm1:com2:2} into \eqref{eq:thm1:com2:1} yields
\begin{align}
  &W_1(\Phi_n(\mathcal{Q}_{n,m}^{(t)}), \Phi_n(\mathcal{Q}_n^{(t)}))\nonumber\\
  &\leq W_1(\mathcal{Q}_{n,m}^{(0)}, \mathcal{Q}_n^{(0)})  \exp {(b^{(t)} (A+(c_1 m_{\mathcal{Q}_{n,m}^{(0)}, \Pi} + c_2 m_{\mathcal{Q}_n^{(0)}, \Pi}) \exp{(Cb^{(t)})}))}.  \label{thm1pf_second}
\end{align}

Third, by Theorem 1 of \cite{shi2024finite}, we have
\begin{equation} \label{thm1pf_third}
   W_1(\mathcal{Q}^{(t+1)}_{n}, \mathcal{Q}^{(t+1)}_\infty) \leq W_1(\mathcal{Q}_{n}^{(0)}, \mathcal{Q}_\infty^{(0)})  \exp {(b^{(t)} (A+(c_1 m_{\mathcal{Q}_{n}^{(0)}, \Pi} + c_2 m_{\mathcal{Q}_\infty^{(0)}, \Pi}) \exp{(Cb^{(t)})}))}. 
\end{equation}

By combining \eqref{thm1pf_first}, \eqref{thm1pf_second}, and \eqref{thm1pf_third}, we have 
\begin{align*}
  &W_1(\mathcal{Q}^{(t)}_{n,m}, \mathcal{Q}^{(t)}_\infty)\\ &\leq W_1(\mathcal{Q}_{n,m}^{(0)}, \mathcal{Q}_n^{(0)}) \exp{(b^{(t-1)} (A+(c_1 m_{\mathcal{Q}_{n,m}^{(0)}, \Pi} + c_2 m_{\mathcal{Q}_n^{(0)}, \Pi}) \exp{(Cb^{(t-1)})}))}\\
  &+ W_1(\mathcal{Q}_{n}^{(0)}, \mathcal{Q}_\infty^{(0)})  \exp {(b^{(t-1)} (A+(c_1 m_{\mathcal{Q}_{n}^{(0)}, \Pi} + c_2 m_{\mathcal{Q}_\infty^{(0)}, \Pi}) \exp{(Cb^{(t-1)})}))} + \mathcal{O}(m^{-1/2})\\
  &\leq (W_1(\mathcal{Q}_{n,m}^{(0)}, \mathcal{Q}_n^{(0)})+W_1(\mathcal{Q}_n^{(0)}, \mathcal{Q}_\infty^{(0)}))  \exp {(b^{(t-1)} (A+B \exp{(Cb^{(t-1)})}))} + \mathcal{O}(m^{-1/2}),  
\end{align*}
where $B=c_1 (m_{\mathcal{Q}_{n,m}^{(0)}, \Pi}+m_{\mathcal{Q}_{n}^{(0)}, \Pi}) + c_2 (m_{\mathcal{Q}_n^{(0)}, \Pi}+m_{\mathcal{Q}_{\infty}^{(0)}, \Pi})$.    
\end{proof}

According to Theorem~\ref{Theorem 1}, the 1-Wasserstein distance between the transformed distributions of MC-SVGD and continuous SVGD is bounded for any iteration $t$. Note that Theorem~\ref{Theorem 1} accounts for both finite particle error and MC approximation error, unlike Theorem 1 in \cite{shi2024finite} which examines only the finite particle error of SVGD. Combining Theorem~\ref{Theorem 1} with Lemma 4 of \cite{shi2024finite}, which derives a bound on KSD based on the 1-Wasserstein distance, Theorem~\ref{Theorem 2} shows that the KSD between the transformed distributions of MC-SVGD and continuous SVGD is bounded for any iteration $t$.

\begin{stheorem} [KSD between MC-SVGD and continuous SVGD]
 \label{Theorem 2}
  Suppose Assumptions 1, 2, and 3 hold. For any $\mathcal{Q}_{n,m}^{(0)}, \mathcal{Q}_n^{(0)}, \mathcal{Q}_\infty^{(0)} \in \mathbf{\Pi}_1$, the outputs $\mathcal{Q}_{n,m}^{(t)}$, $\mathcal{Q}_n^{(t)}$, and $\mathcal{Q}_\infty^{(t)}$ satisfy
      $$\text{KSD}_{\Pi}(\mathcal{Q}_{n,m}^{(t)}, \mathcal{Q}_{\infty}^{(t)}) \leq a^{(t-1)} + \mathcal{O}(m^{-1/2})$$ \begin{align*}
      \text{for} \quad a^{(t-1)} &\triangleq (\kappa_1 L + \kappa_2 d) ((w_{m,n}+w_{n,\infty}) \exp{(b^{(t-1)} (A+B \exp{(Cb^{(t-1)})})))}\\
  &+\kappa_1d^{1/4} L\sqrt{2M_{\mathcal{Q}_\infty^{(0)}, \Pi}(w_{m,n}+w_{n,\infty})}\exp(b^{(t-1)}(2C+A+B \exp(Cb^{(t-1)}))/2),
  \end{align*}
  where $w_{m,n} \triangleq W_1(\mathcal{Q}_{n,m}^{(0)}, \mathcal{Q}_n^{(0)})$, $w_{n,\infty} \triangleq W_1(\mathcal{Q}_{n}^{(0)}, \mathcal{Q}_\infty^{(0)})$, and $A,B,C$ defined as in Theorem \ref{Theorem 1}.
\end{stheorem}

\begin{proof}
According to Lemma 4 of \cite{shi2024finite}, the KSD between MC-SVGD and continuous SVGD is bounded by 1-Wasserstein distance as follows:
    \begin{equation} \label{thm2pf_lemma4}
      \text{KSD}_{\Pi}(\mathcal{Q}_{n,m}^{(t)}, \mathcal{Q}_\infty^{(t)}) \leq (\kappa_1 L + \kappa_2 d) W_1(\mathcal{Q}_{n,m}^{(t)}, \mathcal{Q}_\infty^{(t)})+\kappa_1 d^{1/4} L\sqrt{2M_{\mathcal{Q}_\infty^{(t)},\Pi}W_1(\mathcal{Q}_{n,m}^{(t)}, \mathcal{Q}_\infty^{(t)})}.  
    \end{equation}
    By Theorem~\ref{Theorem 1}, we can bound the 1-Wasserstein distance as follows:  
    \begin{align} 
      &W_1(\mathcal{Q}^{(t)}_{n,m}, \mathcal{Q}^{(t)}_\infty) \nonumber \\
      &\leq (W_1(\mathcal{Q}_{n,m}^{(0)}, \mathcal{Q}_n^{(0)})+W_1(\mathcal{Q}_n^{(0)}, \mathcal{Q}_\infty^{(0)}))  \exp {(b^{(t-1)} (A+B \exp{(Cb^{(t-1)})}))} + \mathcal{O}(m^{-1/2}) \label{thm2pf_thm1}.   
    \end{align}
    By Lemma 3 of \cite{shi2024finite}, the second absolute moment of $\mathcal{Q}_\infty^{(t)}$ is shown to be bounded by that of $\mathcal{Q}_\infty^{(0)}$ as
    \begin{equation} \label{thm2pf_lem3shi}
      M_{\mathcal{Q}_\infty^{(t)}, \Pi} \leq M_{\mathcal{Q}_\infty^{(0)}, \Pi} \prod_{r=0}^{t-1} (1+\epsilon^{(r)} C)^2 \leq M_{\mathcal{Q}_\infty^{(0)}, \Pi} \exp(2C b^{(t-1)}). 
    \end{equation}
    Combining \eqref{thm2pf_lemma4}, \eqref{thm2pf_thm1} and \eqref{thm2pf_lem3shi}, we derive 
    \begin{align*}
      &\text{KSD}_{\Pi}(\mathcal{Q}_{n,m}^{(t)}, \mathcal{Q}_\infty^{(t)})\\
      &\leq (\kappa_1 L + \kappa_2 d) ((w_{m,n}+w_{n,\infty}) \exp{(b^{(t-1)} (A+B \exp{(Cb^{(t-1)})})))}\\
  &\quad +\kappa_1d^{1/4} L\sqrt{2M_{\mathcal{Q}_\infty^{(0)}, \Pi}(w_{m,n}+w_{n,\infty})}\exp(b^{(t-1)}(2C+A+B \exp(Cb^{(t-1)}))/2)\\
  &\quad +\mathcal{O}(m^{-1/2}),
  \end{align*}
where $w_{m,n} \triangleq W_1(\mathcal{Q}_{n,m}^{(0)}, \mathcal{Q}_n^{(0)})$ and $w_{n,\infty} \triangleq W_1(\mathcal{Q}_{n}^{(0)}, \mathcal{Q}_\infty^{(0)})$.    
\end{proof}

Building on Theorem~\ref{Theorem 2} and Corollary 1 in \cite{shi2024finite}, which establish the connection between the KSD and the KL divergence for reference and target distributions, we present Theorem~1 in the main manuscript.

\noindent\textbf{Theorem 1} [A convergence rate of MC-SVGD] 
\textit{Let $\Pi$ satisfies Talagrand's $T_1$ inequality from Definition 22.1 in \cite{villani2009optimal}. Suppose Assumptions 1, 2, and 3 hold, fix any $\mathcal{Q}_{\infty}^{(0)}$ which is absolutely continuous with respect to $\Pi$ and $\mathcal{Q}_{n,m}^{(0)}, \mathcal{Q}_n^{(0)}, \mathcal{Q}_{\infty}^{(0)} \in \mathbf{\Pi}_1$. Then the outputs $\mathcal{Q}_{n,m}^{(t)}$ satisfy
    \begin{equation} 
      \begin{split}
          \min_{0 \leq r \leq t} \text{KSD}_{\Pi}(\mathcal{Q}_{n,m}^{(t)}, \Pi)
    & =\mathcal{O}((\log \log (n))^{-1/2}) + \mathcal{O}(m^{-1/2}),
      \end{split}
  \end{equation}
with probability at least $1-c\delta$ for a universal constant $c>0$.}

\begin{proof}
According to Theorem~\ref{Theorem 2}, we have
       $$ |\text{KSD}_{\Pi}(\mathcal{Q}_{n,m}^{(t)}, \Pi) - \text{KSD}_{\Pi}(\mathcal{Q}_\infty^{(t)}, \Pi)| \leq \text{KSD}_{\Pi}(\mathcal{Q}_{n,m}^{(t)}, \mathcal{Q}_\infty^{(t)})\leq a^{(t-1)} + \mathcal{O}(m^{-1/2})$$
    for each $t$. By Corollary 1 of \cite{shi2024finite}, if $\max_{0 \leq r <t} \epsilon^{(r)} \leq \epsilon^{(t)} \triangleq R_{\alpha,1}$, $\max_{0 \leq r \leq t} \epsilon^{(r)} \leq R_{\alpha, 2}$ for some $\alpha>1$, the KSD between $\mathcal{Q}_{n,m}^{(r)}$ and target $\Pi$ can be bounded by the KL divergence between $\mathcal{Q}_{\infty}^{(0)}$ and $\Pi$ as follows:
    \begin{equation} \label{thm3pf_cor1shi}
     \begin{split}
        \sum_{r=0}^t \xi^{(r)} (\text{KSD}_{\Pi}(\mathcal{Q}_{n,m}^{(r)}, \Pi) - (a^{(t-1)} +\mathcal{O}(m^{-1/2})))^2 & \leq \sum_{r=0}^t \xi^{(r)} \text{KSD}_{\Pi}(\mathcal{Q}_{\infty}^{(r)}, \Pi)^2\\
        & \leq \frac{2}{R_{\alpha, 1} + b^{(t-1)}} KL(\mathcal{Q}_{\infty}^{(0)} \| \Pi),
     \end{split}  
    \end{equation}
   where $R_{\alpha, p} \triangleq \min \left(\frac{p}{\kappa_1^2(L+\alpha^2)}, (\alpha-1)(1+Lm_{\mathcal{Q}_n^{(0)}, \boldsymbol{\theta}^*}+2L\sqrt{\frac{2}{\lambda}KL(\mathcal{Q}_{\infty}^{(0)} \| \Pi)})\right) \text{for } p \in \{1,2\}, \lambda>0$, and $\xi^{(t)} \triangleq \frac{c(\epsilon^{(t)})}{\sum_{r=0}^t c(\epsilon^{(t)})}$ for $c(\epsilon) \triangleq \epsilon\left(1-\frac{\kappa_1^2(L+\alpha^2)}{2} \epsilon\right)$.
    By Jensen's inequality, we have
    \begin{equation} \label{thm3pf_jensen}
     \begin{split}
      &\sum_{r=0}^t \xi^{(r)} (\text{KSD}_{\Pi}(\mathcal{Q}_{n,m}^{(r)}, \Pi) - (a^{(t-1)} +\mathcal{O}(m^{-1/2})))^2\\
      &\geq \left(\sum_{r=0}^t \xi^{(r)} \text{KSD}_{\Pi}(\mathcal{Q}_{n,m}^{(r)}, \Pi)- \sum_{r=0}^t \xi^{(r)} (a^{(t-1)} +\mathcal{O}(m^{-1/2}))\right)^2. 
     \end{split}
    \end{equation}
     Combining \eqref{thm3pf_cor1shi} and \eqref{thm3pf_jensen}, we have
     $$\sum_{r=0}^t \xi^{(r)} \text{KSD}_{\Pi}(\mathcal{Q}_{n,m}^{(r)}, \Pi) \leq \sum_{r=0}^t \xi^{(r)}(a^{(t-1)} +\mathcal{O}(m^{-1/2})) + \sqrt{\frac{2}{R_{\alpha,1}+b^{(t-1)}}KL(\mathcal{Q}_n^{(0)} \|\Pi)}.$$
    Given $\sum_{r=0}^t \xi^{(r)}(a^{(t-1)} +\mathcal{O}(m^{-1/2})) \leq \max_{0 \leq r \leq t} (a^{(t-1)} +\mathcal{O}(m^{-1/2})) = a^{(t-1)} +\mathcal{O}(m^{-1/2})$, the proof is concluded.

Let $(\overline{w}, \overline{A}, \overline{B}, \overline{C})$ be any upper bounds on $(\max(w_{m,n}, w_{n,\infty}), A, B, C)$ respectively, and define the growth functions
  $$ \phi(w) \triangleq \log \log \left(e^e +\frac{1}{w}\right), \quad  \zeta_{\overline{B}, \overline{C}} (x,y,\beta) \triangleq \frac{1}{\overline{C}} \log \left(\frac{1}{\overline{B}} \max \left(\overline{B} , \frac{1}{\beta} \log \frac{1}{x} -y\right)\right).$$ And the sum of step size $b^{(t-1)} = \sum_{r=0}^{t-1} \epsilon^{(t)} = s_n^*$ for
  $$s_n^* \triangleq \min(\zeta_{\overline{B}, \overline{C}} (\overline{w} \sqrt{\phi(\overline{w})}, \overline{A}, \beta_1), \zeta_{\overline{B}, \overline{C}} (\overline{w} \phi(\overline{w}), \overline{A}+2\overline{C}, \beta_2))=\min(b_1^{(t-1)}, b_2^{(t-1)}),$$
  $$\beta_1 \triangleq \max (1, \zeta_{\overline{B}, \overline{C}} (\overline{w} \sqrt{\phi(\overline{w})}, \overline{A}, 1)), \quad \beta_2 \triangleq \max(1, \zeta_{\overline{B}, \overline{C}} (\overline{w} \phi(\overline{w}), \overline{A}+2\overline{C}, 1)).$$ Let $b^{(t-1)}_1=\zeta_{\overline{B}, \overline{C}}(\overline{w} \sqrt{\phi\left(\overline{w}\right)}, \overline{A}, \beta_{1})$ and $b^{(t-1)}_2=\zeta_{\overline{B}, \overline{C}}\left(\overline{w} \phi\left(\overline{w}\right), \overline{A}+2 \overline{C}, \beta_{2}\right)$ such that $b^{(t-1)}=\min \{b^{(t-1)}_1, b^{(t-1)}_2\}$. Since $\beta_{1}$, $\beta_{2}$, and $\phi\left(\overline{w}\right)$ are greater than or equal to 1, we have
\begin{align*}
\beta_{1} &= \max \left\{1, \frac{1}{C} \log \left(\frac{1}{B}\left(\log \frac{1}{\overline{w} \sqrt{\phi\left(\overline{w}\right)}}-\overline{A}\right)\right)\right\}\\
&\leq \max \left\{1, \frac{1}{C} \log \left(\frac{1}{B}\left(\log \frac{1}{\overline{w} \sqrt{\phi\left(\overline{w}\right)}}\right)\right)\right\} \\
& \leq \max \left\{1, \frac{1}{C} \log \left(\frac{1}{B}\left(\log \frac{1}{\overline{w}}\right)\right)\right\},
\end{align*}
and
\begin{align*}
\beta_{2} &=\max \left\{1, \frac{1}{C} \log \left(\frac{1}{B}\left(\log \frac{1}{\overline{w} \phi\left(\overline{w}\right)}-\overline{A}-2 \overline{C}\right)\right)\right\}\\ 
&\leq \max \left\{1, \frac{1}{C} \log \left(\frac{1}{B}\left(\log \frac{1}{\overline{w} \phi\left(\overline{w}\right)}\right)\right)\right\} \\
& \leq \max \left\{1, \frac{1}{C} \log \left(\frac{1}{B}\left(\log \frac{1}{\overline{w}}\right)\right)\right\}.
\end{align*}
Since $\phi\left(\overline{w}\right) \geq 1$, the lower bounds of $b^{(t-1)}_1$ and $b^{(t-1)}_2$ can be given by
\begin{equation} \label{stepsize_lowerbound}
    \begin{split}
        & b^{(t-1)}_1\geq \frac{1}{\overline{C}} \log \left(\frac{1}{\overline{B}}\left(\frac{\log \frac{1}{\overline{w} \phi\left(\overline{w}\right)}}{\max \left\{1, \frac{1}{\overline{C}} \log \left(\frac{1}{\overline{B}}\left(\log \frac{1}{\overline{w}}\right)\right)\right\}}-\overline{A}-2 \overline{C}\right)\right) \\
&b^{(t-1)}_2\geq \frac{1}{\overline{C}} \log \left(\frac{1}{\overline{B}}\left(\frac{\log \frac{1}{\overline{w} \phi\left(\overline{w}\right)}}{\max \left\{1, \frac{1}{\overline{C}} \log \left(\frac{1}{\overline{B}}\left(\log \frac{1}{\overline{w}}\right)\right)\right\}}-\overline{A}-2 \overline{C}\right)\right).
    \end{split}
\end{equation}

We derive the error bound by considering two distinct cases for the range of $b^{(t-1)}$. If $b^{(t-1)}=0$, the error bound follows directly from Theorem~1 of this manuscript.

If $b^{(t-1)}>0$, Lemma 6 of \cite{shi2024finite} implies that $b^{(t-1)}_1 \leq \beta_{1}$ and
$$
b^{(t-1)}_1 \leq \psi_{\overline{B}, \overline{C}}\left(\overline{w} \sqrt{\phi\left(\overline{w}\right)}, \overline{A}, b^{(t-1)}_1\right)
$$
since $\zeta_{\overline{B}, \overline{C}}$ is non-increasing.
Rearranging the terms and noting that
$$
\overline{B}<\frac{1}{\beta_{1}} \log \frac{1}{\overline{w} \sqrt{\phi\left(\overline{w}\right)}}-\overline{A} \leq \frac{1}{b^{(t-1)}_1} \log \frac{1}{\overline{w} \sqrt{\phi\left(\overline{w}\right)}}-\overline{A}
$$
and $b^{(t-1)}_1 \geq b^{(t-1)}>0$, we obtain
\begin{equation*}
\overline{w} \exp \left(b^{(t-1)}_1\left(\overline{A}+\overline{B} \exp \left(\overline{C} b^{(t-1)}_1\right)\right)\right) \leq \frac{1}{\sqrt{\phi\left(\overline{w}\right)}}. 
\end{equation*}
Similarly, we have $b^{(t-1)}_2 \leq \beta_2$ and the following inequalities hold:
\begin{equation*}
\sqrt{w_{m,n}} \exp \left(b^{(t-1)}_2\left(2 \overline{C}+\overline{A}+\overline{B} \exp \left(\overline{C} b^{(t-1)}_2\right)\right) / 2\right) \leq \frac{1}{\sqrt{\phi\left(\overline{w}\right)}}, 
\end{equation*}
\begin{equation*}
\sqrt{w_{n,\infty}} \exp \left(b^{(t-1)}_2\left(2 \overline{C}+\overline{A}+\overline{B} \exp \left(\overline{C} b^{(t-1)}_2\right)\right) / 2\right) \leq \frac{1}{\sqrt{\phi\left(\overline{w}\right)}}. 
\end{equation*}
Since $b^{(t)}=\min \{b^{(t-1)}_1, b^{(t-1)}_2\}$, the inequalities hold when substituting $b^{(t)}$ for $b^{(t-1)}_1$ and $b^{(t-1)}_2$. Given that the error term $a^{(t)}$ is non-decreasing with respect to each of $\left( w_{m,\infty}, A, B, C\right)$, we conclude that
$$
a^{(t)} \leq\left(\kappa_{1} L+\kappa_{2} d+\kappa_{1} d^{1 / 4} L \sqrt{2 M_{\mathcal{Q}_\infty^{(0)}}}\right) / \sqrt{\phi\left(\overline{w}\right)}.
$$
Under the conditions
$$e^{-(\overline{B}+\overline{A}+2\overline{C})}> \overline{w} \phi (\overline{w}) \quad \text{and} \quad \overline{B}^{(\overline{B}+\overline{A}+2\overline{C})/\overline{C}} > \overline{w} \phi(\overline{w})(\log(1/\overline{w}))^{(\overline{B}+\overline{A}+2\overline{C})/\overline{C}},$$ which hold for a sufficiently small $\overline{w}$, we obtain $$\min_{0 \leq r \leq t} \text{KSD}_{\Pi}(\mathcal{Q}_{n,m}^{(t)}, \Pi) \leq \mathcal{O}\left(\left(\log \log (e^e + \frac{1}{\overline{w}})\right)^{-1/2}\right)+\mathcal{O}(m^{-1/2}).$$
According to Theorem 3.1 of \cite{lei2020convergence} and Markov's inequality, we have
$$W_1(\mathcal{Q}_{n,m}^{(0)}, \mathcal{Q}_{\infty}^{(0)}) \leq \mathbb{E}\left[W_1(\mathcal{Q}_{n,m}^{(0)}, \mathcal{Q}_{\infty}^{(0)})\right]/(c\delta) \leq \frac{M_{\mathcal{Q}_\infty^{(0)}} \log(n) ^ {\mathbb{I}\left[d=2\right]}}{\delta n^{1/\max(2,d)}} $$
with probability at least $1-c\delta$ for a constant $c>0$.
Similarly, $W_1(\mathcal{Q}_{n}^{(0)}, \mathcal{Q}_{\infty}^{(0)})$ has the same upper bound with probability at least $1-c\delta$ for a constant $c>0$. By the triangle inequality, we derive $$W_1(\mathcal{Q}_{n,m}^{(0)}, \mathcal{Q}_{n}^{(0)}) \leq \frac{2M_{\mathcal{Q}_\infty^{(0)}} \log(n) ^ {\mathbb{I}\left[d=2\right]}}{\delta n^{1/\max(2,d)}}.$$ Hence, with this choice of $\overline{w}\triangleq \frac{2M_{\mathcal{Q}_\infty^{(0)}} \log(n) ^ {\mathbb{I}\left[d=2\right]}}{\delta n^{1/\max(2,d)}},$
$$\min_{0 \leq r \leq t} \text{KSD}_{\Pi} (\mathcal{Q}_{n,m}^{(t)}, \Pi) = \mathcal{O}((\log \log (n))^{-1/2}) + \mathcal{O}(m^{-1/2})$$
with probability at least $1-c\delta$.  
\end{proof}

For the sum of the step size sequence, we have $s_n^*= \min(b_1^{(t-1)}, b_2^{(t-1)}) \leq \min(\beta_1, \beta_2)$ from the proof of Theorem~1, where the smallest possible values for $\beta_1$ and $\beta_2$ are 1. 

\section{Sensitivity Analysis}
\label{appendix_b}
To provide practical guidelines for implementation, we conduct a sensitivity analysis on the following tuning components: (1) the number of particles ($n$), (2) the number of Monte Carlo samples ($m$), (3) the ESS threshold, and (4) the number of algorithm iterations. The experiments are conducted using the exponential random graph models (ERGMs) presented in the manuscript, which is the most challenging of the discussed examples due to its high-dimensional parameter space.


\newpage
\subsection{Number of Particles}
Experiments are conducted based on the observation that higher parameter dimensions require a greater number of particles. We test the number of particles per dimension using $10d$, $30d$, and $50d$ particles, where the dimension $d=10$ for the ERGM example. The corresponding computation times are 0.81, 4.47, and 10.75 minutes, respectively. We treat the sample from the DMH as the gold standard. Figure~\ref{fig:sensitiv_par} shows MC-SVGD exhibit a noticeable deviation from DMH when $n=10d$, whereas it produces results comparable to DMH when $n\geq 30d$.
\begin{figure}[htbp]
     \centering
     \begin{figure}
         \centering
         \includegraphics[width=\textwidth]{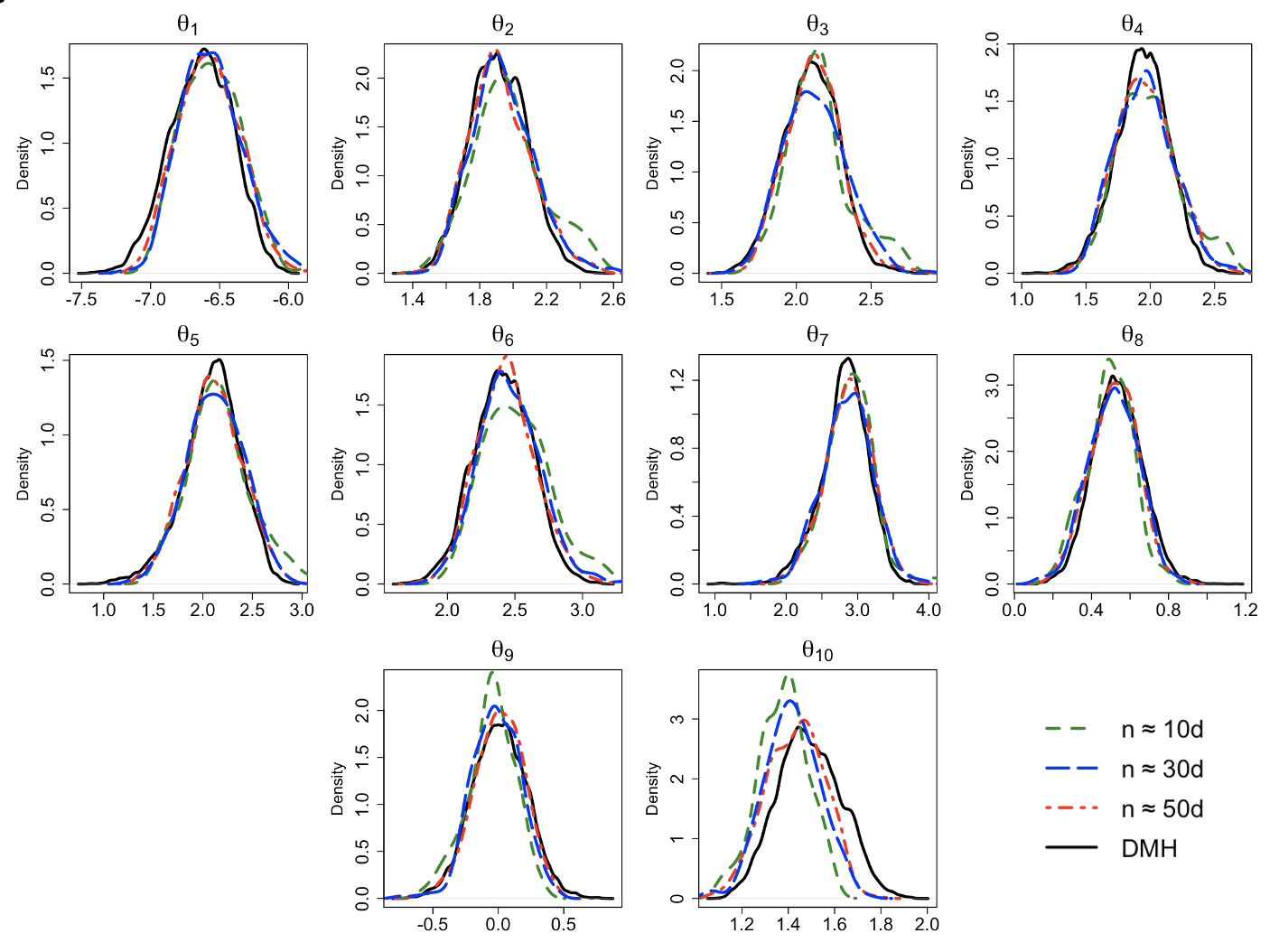}
     \end{figure}
        \caption{Posterior density estimates for the parameters of the ERGM model on the Faux Mesa high school network dataset with various numbers $n$ of particles.}
        \label{fig:sensitiv_par}
\end{figure}

\subsection{Number of Monte Carlo Samples}
Experiments are carried out with a range of different MC sample sizes. 
Specifically, we consider 30, 50, and 100 for the MC sample size $m$. The computation times for these cases are 1.75, 4.47, and 26.14 minutes, respectively. Figure~\ref{fig:sensitiv_mcsamples} demonstrates that the results are consistent across different values of $m$, with all showing similarity to DMH.
\begin{figure}[htbp]
     \centering
     \begin{figure}
         \centering
         \includegraphics[width=\textwidth]{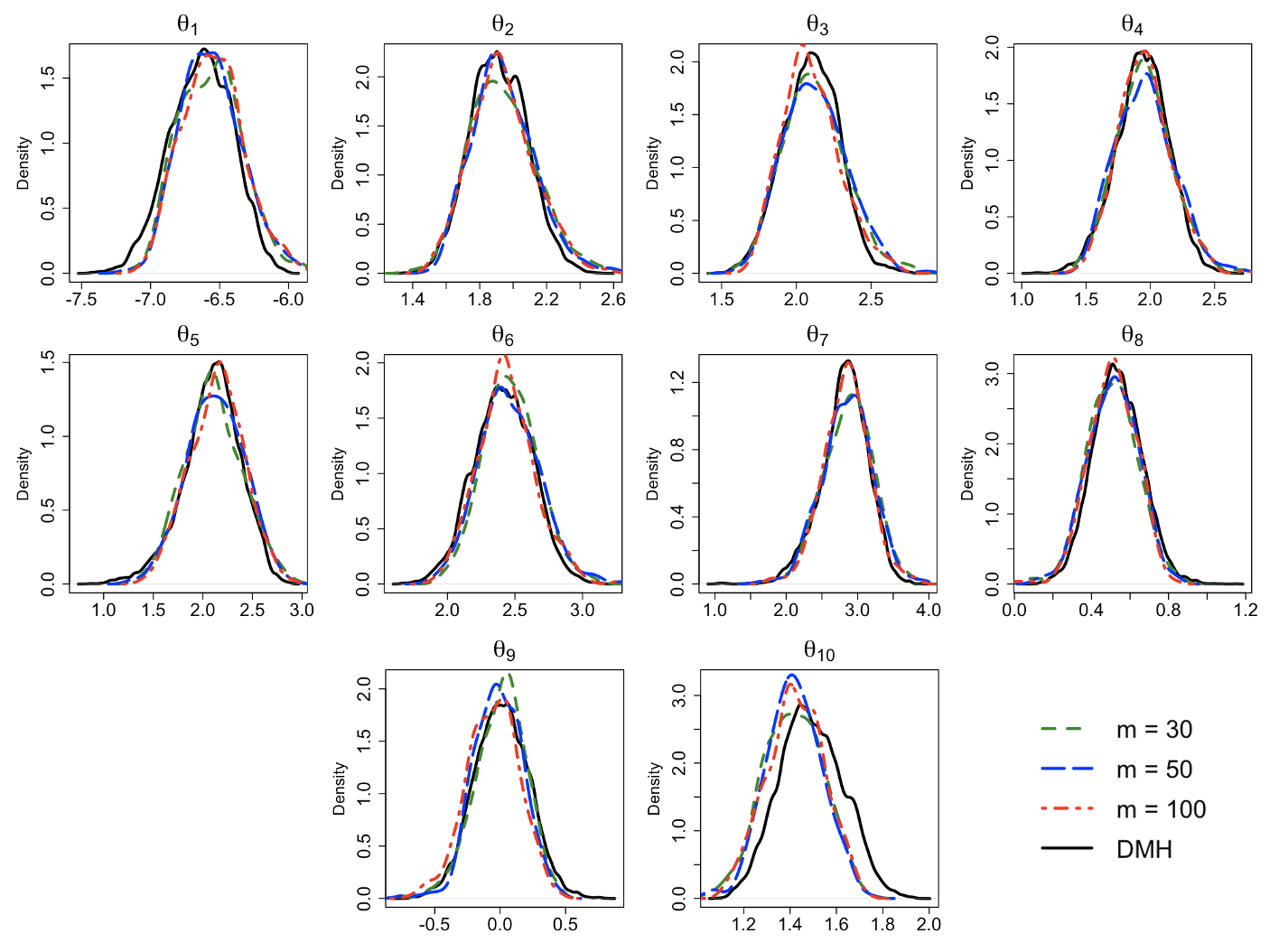}
     \end{figure}
        \caption{Posterior density estimates for the parameters of the ERGM model on the Faux Mesa high school network dataset with various choices $m$ of Monte Carlo sample size.}
        \label{fig:sensitiv_mcsamples}
\end{figure}

\subsection{ESS Threshold}
Here we perform a sensitivity analysis on the ESS threshold, beginning with $m/3$, as suggested by \cite{tan2020bayesian}, and exploring higher (more conservative) threshold values. The thresholds examined include $m/3$, $m/2$, and $m/1.5$. The associated computation times are 1.96, 2.95, and 4.47 minutes, respectively. Figure~\ref{fig:sensitiv_threshold} illustrates that the results are generally consistent across the threshold values, except for $\theta_2$ and $\theta_6$, for which higher threshold values yield results that are more similar to DMH.
\begin{figure}[htbp]
     \centering
     \begin{figure}
         \centering
         \includegraphics[width=\textwidth]{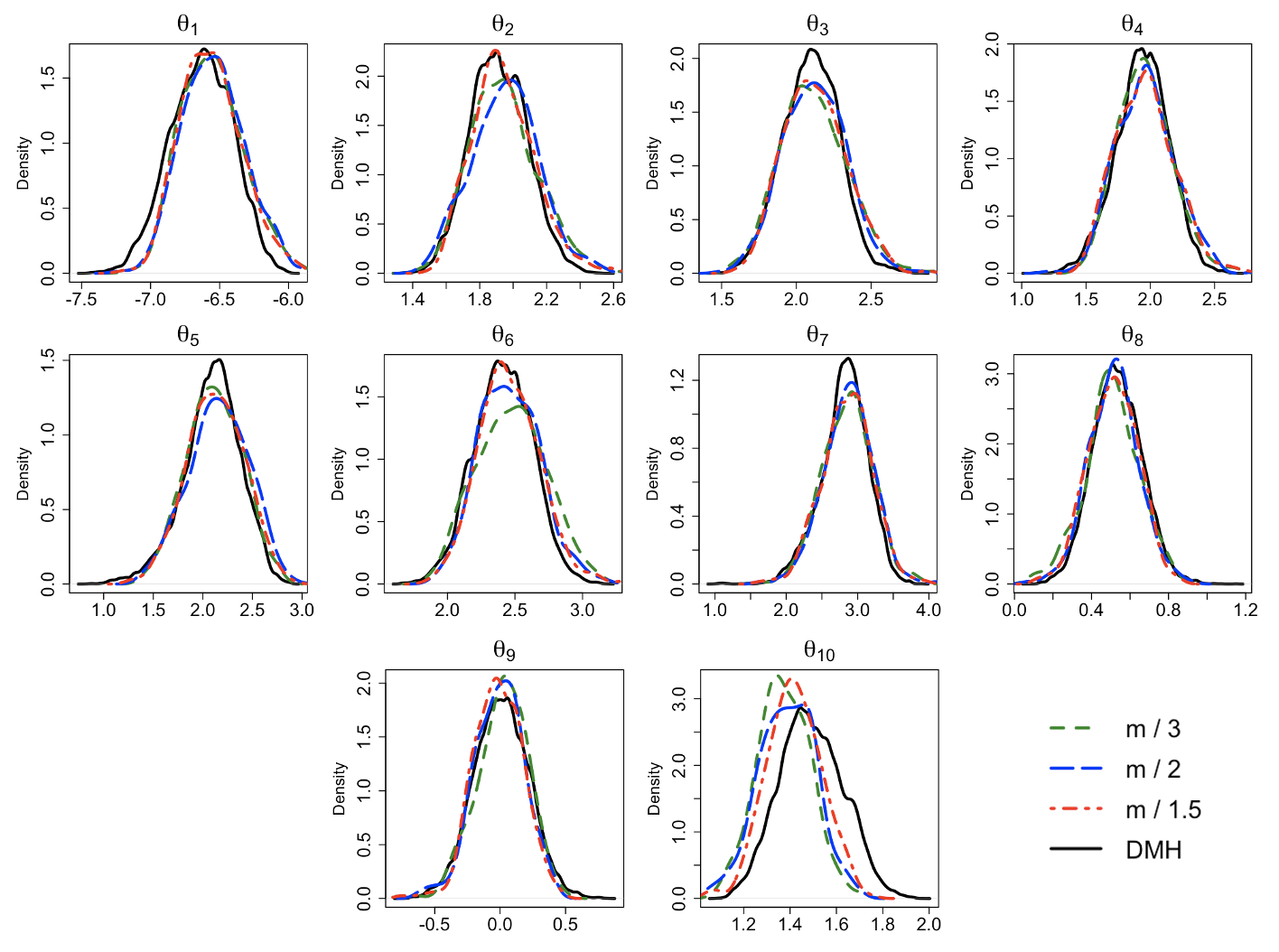}
     \end{figure}
        \caption{Posterior density estimates for parameters of the ERGM model on the Faux Mesa high school network dataset with different values of threshold.}
        \label{fig:sensitiv_threshold}
\end{figure}

\subsection{Number of Iterations}
To assess the number of iterations needed for our MC-SVGD to achieve convergence, we compare results produced by the algorithm with varying numbers $T$ of iterations. 
We specifically consider 100, 300, 500, and 1000 for $T$. The corresponding computation times are 0.57, 2.35, 4.47, and 7.61 minutes, respectively. Figure~\ref{fig:sensitiv_iter} shows that the posterior density estimates remain consistent for $T > 500$.
\begin{figure}[htbp]
     \centering
     \begin{figure}
         \centering
         \includegraphics[width=\textwidth]{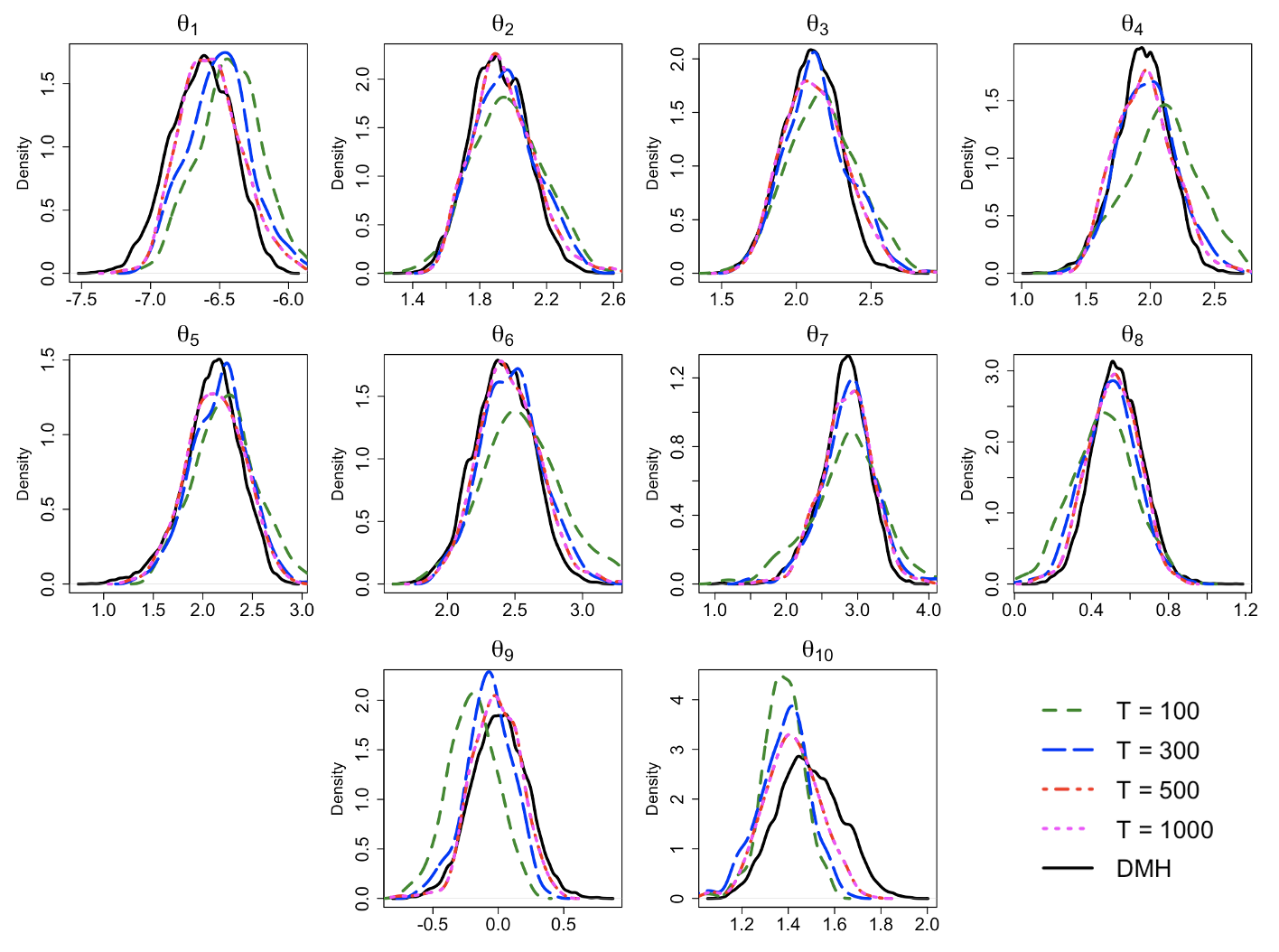}
     \end{figure}
        \caption{Posterior density estimates of parameters for the ERGM model on the Faux Mesa high school network dataset with increasing numbers of iterations.}
        \label{fig:sensitiv_iter}
\end{figure}

To determine an appropriate stopping point for the algorithm, we recommend calculating the KL divergence between $\mathcal{Q}_{n,m}^{(t)}$ and $\mathcal{Q}_{n,m}^{(t-100)}$ at every 100th iteration and halting the algorithm when there are no significant changes in the divergence. Figure~\ref{fig:kl_divergence} shows the KL divergence as the number of iterations increases. The divergence values decrease sharply up to iteration 300 and show minimal change beyond iteration 500, which is consistent with Figure~\ref{fig:sensitiv_iter}.

\begin{figure}[htbp]
     \centering
     \begin{figure}
         \centering
         \includegraphics[width=0.5\textwidth]{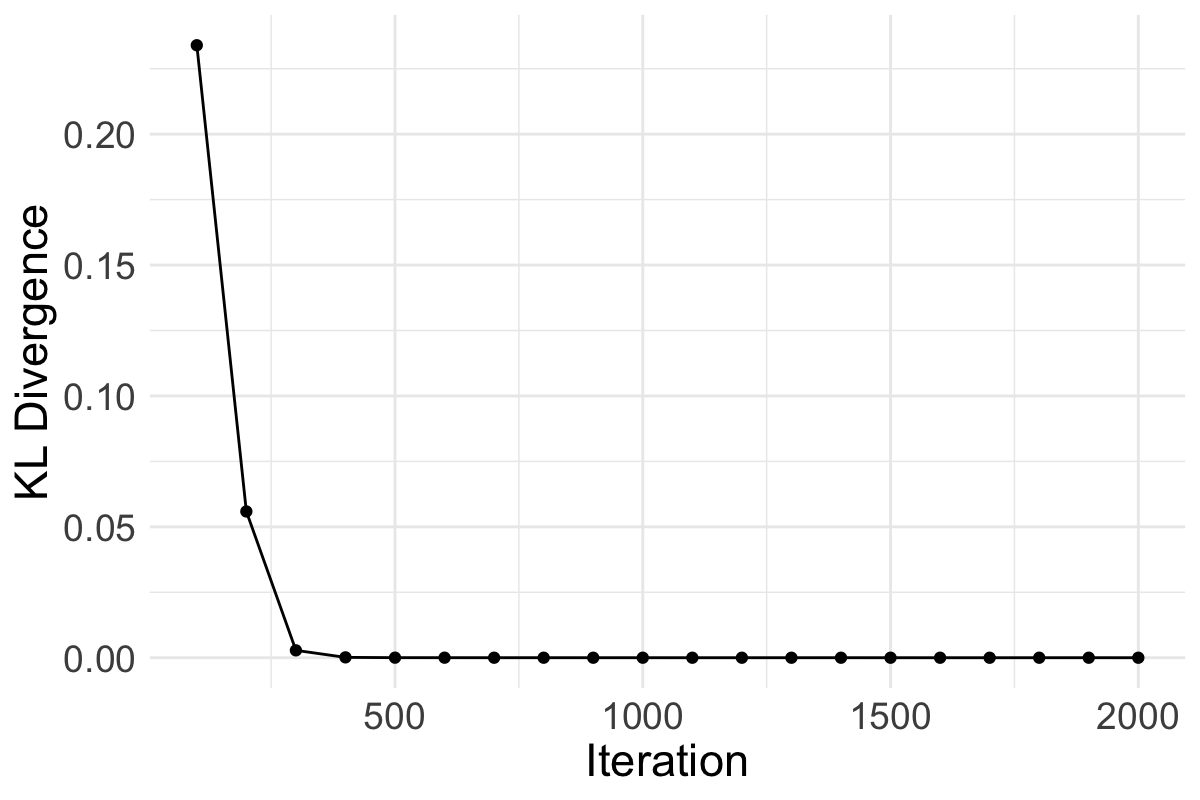}
     \end{figure}
        \caption{KL divergence between sample distributions}
        \label{fig:kl_divergence}
\end{figure}

\section{Derivations for the Gradient}
\label{appendix_c}
In this section, we provide derivations for the gradient of applications in Section~4.
\subsection{A Potts Model}
\begin{align*}
    L(\theta | \textbf{x}) &= \frac{1}{Z(\theta)} \exp\left\{\theta S(\textbf{x})\right\},\\
    h(\textbf{x} | \theta) &= \exp\left\{\theta S(\textbf{x})\right\},\\
    \log h(\textbf{x} | \theta) &= \theta S(\textbf{x}), \\
    \nabla_{\theta} \log h(\textbf{x} | \theta) &= S(\textbf{x}).
\end{align*}

\subsection{A Conway--Maxwell--Poisson Regression Model}
\begin{align*}
    L(\boldsymbol{\theta}|\textbf{y}) &= \frac{\exp \left\{ \nu \sum_{j=1}^{p} \beta_j S_j(\textbf{y}) - \nu S_{p+1}(\textbf{y})\right\}}{\prod_{i=1}^n Z(\eta_i, \nu)},\\
    h(\boldsymbol{\theta} | \textbf{y}) &= \exp \left\{ \nu \sum_{j=1}^{p} \beta_j S_j(\textbf{y}) - \nu S_{p+1}(\textbf{y})\right\},\\
    \log h(\boldsymbol{\theta} | \textbf{y}) &= \nu \sum_{j=1}^{p} \beta_j S_j(\textbf{y}) - \nu S_{p+1}(\textbf{y}),\\
    \nabla_{\beta_j} \log h(\boldsymbol{\theta} | \textbf{y}) &= \nu  S_j(\textbf{y}),\\
    \nabla_{\nu}\log h(\boldsymbol{\theta} | \textbf{y}) &= S_{p+1}(\textbf{y}).
\end{align*}

\subsection{An Exponential Random Graph Model}
\begin{align*}
    L(\boldsymbol{\theta} | \textbf{x}) &= \frac{1}{Z(\boldsymbol{\theta})} \exp\left\{\boldsymbol{\theta} S(\textbf{x})\right\},\\
    h(\textbf{x} | \boldsymbol{\theta}) &= \exp\left\{\boldsymbol{\theta} S(\textbf{x})\right\},\\
    \log h(\textbf{x} | \boldsymbol{\theta}) &= \boldsymbol{\theta} S(\textbf{x}), \\
    \nabla_{\boldsymbol{\theta}} \log h(\textbf{x} | \boldsymbol{\theta}) &= S(\textbf{x}).
\end{align*}

\section{Extra Results}
\label{appendix_d}
In this section, we present the results for parameters that are not provided in the manuscript. 
\subsection{A Conway--Maxwell--Poisson Regression Model}

\begin{table}[htbp]
  \centering
    \begin{tabular}{M{10mm}M{20mm}M{20mm}M{25mm}M{25mm}M{25mm}} 
    \toprule
    \textbf{N} & \textbf{Method} & \textbf{Report} & \multicolumn{1}{M{25mm}}{$\theta_1$} &\multicolumn{1}{M{25mm}}{$\theta_2$} &\multicolumn{1}{M{25mm}}{$\theta_3$}\\ \midrule
    
    \multirow{6}{*}{15} & \multirow{3}{*}{Exchange} & Mean  &  1.05 & 1.05 & 0.04   \\ 
    & & 95$\%$HPD  &(0.87, 1.25)&(0.86, 1.24)& (-0.12, 0.21)\\
    & & Time (sec) &30& &  \\ \cmidrule{2-6}
    
    & \multirow{3}{*}{MC-SVGD} & Mean  &  1.05 & 1.05 & 0.04\\ 
    & & 95$\%$HPD  &(0.85, 1.21)&(0.89, 1.24)& (-0.15, 0.19) \\
    & & Time (sec) &5.08& &  \\ \midrule
    
    \multirow{6}{*}{20} & \multirow{3}{*}{Exchange} & Mean  &  0.93 & 0.98 & 0.17   \\ 
    & & 95$\%$HPD  &(0.79, 1.07)&(0.84, 1.13)& (0.04, 0.30)\\
    & & Time (sec) &51.6& &  \\ \cmidrule{2-6}
    
    & \multirow{3}{*}{MC-SVGD} & Mean  & 0.93 & 0.98 & 0.17\\ 
    & & 95$\%$HPD  &(0.77, 1.05)&(0.81, 1.10)& (0.04, 0.29) \\
    & & Time (sec) &7.37& &  \\ \midrule
    
    \multirow{6}{*}{50} & \multirow{3}{*}{Exchange} & Mean  &  1.01 & 0.95 & 0.12   \\ 
    & & 95$\%$HPD  &(0.96, 1.08)&(0.89, 1.01)& (0.06, 0.17)\\
    & & Time (sec) &258& &  \\ \cmidrule{2-6}
    
    & \multirow{3}{*}{MC-SVGD} & Mean  &  1.01 & 0.95 & 0.12 \\ 
    & & 95$\%$HPD  &(0.96, 1.07)&(0.89, 0.99)& (0.07, 0.16) \\
    & & Time (sec) &38.56& &  \\ \bottomrule
    
    \end{tabular}
    \caption{Posterior inference results of all parameters for the COMP regression model on the simulated data with increasing data size $N$. The simulated true parameter is $\boldsymbol{\theta} = (1,1,0.1)$
    }
    \label{table:comp_all}
\end{table}

\newpage
\subsection{An Exponential Random Graph Model}
\begin{table}[htbp]
  \centering
    \begin{tabular}{M{24mm}M{24mm}M{22mm}M{22mm}M{22mm}M{22mm}} \toprule
    \textbf{DMH} &\multicolumn{1}{M{20mm}}{$\boldsymbol{\theta}_1$} &\multicolumn{1}{M{20mm}}{$\boldsymbol{\theta}_2$} &\multicolumn{1}{M{20mm}}{$\boldsymbol{\theta}_3$} &\multicolumn{1}{M{20mm}}{$\boldsymbol{\theta}_4$} &\multicolumn{1}{M{20mm}}{$\boldsymbol{\theta}_5$}\\ \midrule
        Mean  &  -6.63 & 1.91& 2.10  & 1.94 &2.09  \\ 
        95$\%$HPD  &(-7.06, -6.20)&(1.58, 2.25)& (1.72, 2.44) & (1.51, 2.32) &(1.50, 2.63)\\
\cmidrule(l){2-6}
\textbf{} &\multicolumn{1}{M{20mm}}{$\boldsymbol{\theta}_6$} &\multicolumn{1}{M{20mm}}{$\boldsymbol{\theta}_7$} &\multicolumn{1}{M{20mm}}{$\boldsymbol{\theta}_8$} &\multicolumn{1}{M{20mm}}{$\boldsymbol{\theta}_9$} &\multicolumn{1}{M{20mm}}{$\boldsymbol{\theta}_{10}$}\\ \cmidrule(l){2-6}
        Mean  &  2.41&2.81& 0.53  & 0.01 &1.49  \\ 
        95$\%$HPD  &(2.00, 2.84)&(2.13, 3.42)& (0.28, 0.78) & (-0.40, 0.42) &(1.23, 1.75)\\
        Time (min) &71.85& &  &  & \\   \toprule
    \textbf{SVI} &\multicolumn{1}{M{20mm}}{$\boldsymbol{\theta}_1$} &\multicolumn{1}{M{20mm}}{$\boldsymbol{\theta}_2$} &\multicolumn{1}{M{20mm}}{$\boldsymbol{\theta}_3$} &\multicolumn{1}{M{20mm}}{$\boldsymbol{\theta}_4$} &\multicolumn{1}{M{20mm}}{$\boldsymbol{\theta}_5$}\\ \midrule
        Mean  &  -6.72&1.93& 2.10  & 1.95 &2.13  \\ 
        95$\%$HPD  &(-7.15, -6.29)&(1.60, 2.27)& (1.73, 2.48) & (1.55, 2.33) &(1.59, 2.67)\\
 \cmidrule(l){2-6}
\textbf{} &\multicolumn{1}{M{20mm}}{$\boldsymbol{\theta}_6$} &\multicolumn{1}{M{20mm}}{$\boldsymbol{\theta}_7$} &\multicolumn{1}{M{20mm}}{$\boldsymbol{\theta}_8$} &\multicolumn{1}{M{20mm}}{$\boldsymbol{\theta}_9$} &\multicolumn{1}{M{20mm}}{$\boldsymbol{\theta}_{10}$}\\ \cmidrule(l){2-6}
        Mean  &  2.43&2.93& 0.55  & 0.07 &1.53  \\ 
        95$\%$HPD  &(2.00, 2.85)&(2.28, 3.57)& (0.29, 0.80) & (-0.34, 0.48) &(1.25, 1.81)\\
        Time (min) &19.83& &   &  & \\   \toprule
        \textbf{MC-SVGD} &\multicolumn{1}{M{20mm}}{$\boldsymbol{\theta}_1$} &\multicolumn{1}{M{20mm}}{$\boldsymbol{\theta}_2$} &\multicolumn{1}{M{20mm}}{$\boldsymbol{\theta}_3$} &\multicolumn{1}{M{20mm}}{$\boldsymbol{\theta}_4$} &\multicolumn{1}{M{20mm}}{$\boldsymbol{\theta}_5$}\\ \midrule
        Mean  & -6.56 & 1.94 & 2.12 & 1.96 & 2.13\\ 
        95$\%$HPD  &(-6.93, -6.11)&(1.64, 2.30)& (1.76, 2.52) & (1.52, 2.34) &(1.58, 2.69)\\
 \cmidrule(l){2-6}
\textbf{} &\multicolumn{1}{M{20mm}}{$\boldsymbol{\theta}_6$} &\multicolumn{1}{M{20mm}}{$\boldsymbol{\theta}_7$} &\multicolumn{1}{M{20mm}}{$\boldsymbol{\theta}_8$} &\multicolumn{1}{M{20mm}}{$\boldsymbol{\theta}_9$} &\multicolumn{1}{M{20mm}}{$\boldsymbol{\theta}_{10}$}\\ \cmidrule(l){2-6}
        Mean  &  2.46 & 2.85 & 0.51 & -0.01 & 1.42\\ 
        95$\%$HPD  &(2.04, 2.90)&(2.13, 3.43)& (0.28, 0.74) & (-0.33, 0.34) &(1.22, 1.66)\\
        Time (min) &4.47& &  &  & \\   \bottomrule
    \end{tabular}
    \caption{Posterior inference results of all parameters for the ERGM model on the Faux Mesa high school network dataset. 
    }
    \label{tab:ergm_all}
\end{table}

\bibliographystyle{apalike}
\bibliography{ref2}
\end{document}